\documentclass{article}

\usepackage[round]{natbib}

\usepackage{amsmath}
\usepackage{amsthm}
\usepackage{graphicx} 
\usepackage{epstopdf}
\usepackage{subcaption}
\usepackage{amsfonts}
\usepackage{amssymb}
\usepackage{bm,bbold}
\usepackage{mathtools}
\usepackage{etoolbox}
\usepackage[algo2e,linesnumbered,ruled,vlined]{algorithm2e} 
\usepackage{tabularx}
\usepackage{caption}
\usepackage{multirow}
\usepackage{booktabs}

\usepackage{refcount}

\newtheorem{theorem}{Theorem}[section]
\newtheorem{remark}{Remark}[section]
\newtheorem{lemma}{Lemma}[section]

\newtheorem{assumption}{Assumption}[section]

\renewcommand{\bm}{\mathbf}
\newcommand{\err}{\boldsymbol \xi}
\newcommand{\res}{\bm r}
\newcommand{\indset}{\mathcal{S}}
\renewcommand{\S}{\indset}
\newcommand{\norm}[1]{\left\lVert#1\right\rVert}
\newcommand{\abs}[1]{\left| #1\right| }
\newcommand{\innerp}[2]{\left\langle #1, #2 \right\rangle}

\newcommand{\sign}{\operatorname{sign}}
\newcommand{\E}{\textup{E}}

\newcommand{\Cov}{ \textup{Cov}}

\newcommand{\supp}{\operatorname{supp}}
\newcommand{\Nor}{\mathcal N}
\newcommand{\erre}{\tilde\err}
\newcommand{\mumax}{\mu_{\max}}

\newcommand{\vv}{\bm v}

\newcommand{\tmin}{\theta_{\min}}
\newcommand{\tmax}{\theta_{\max}}
\newcommand{\R}{\mathbb{R}}
\newcommand{\rhot}{\rho}
\newcommand{\jc}{j_{\text{center}}}
\newcommand{\tv}{\boldsymbol{\theta}}
\newcommand{\x}{\bm x}
\newcommand{\X}{\bm X}
\newcommand{\y}{\bm y}

\newcommand{\z}{\bm z}

\newcommand{\Mc}{M_c}
\newcommand{\tc}{t_c}
\newcommand{\crit}{{\mbox{\tiny crit}}}
\newcommand{\kmax}{k}

\newcommand{\EM}{\bm A} 
\newcommand{\Ev}{\bm a} 
\newcommand{\Ku}{K_u} 
\newcommand{\Kd}{K_d} 
\newcommand{\lmin}{\lambda_{\min}}
\newcommand{\lmax}{\lambda_{\max}}
\renewcommand{\P}{\bm{P}}
\newcommand{\I}{\bm{I}}

\newcommand{\hatS}{S}
\newcommand{\hats}{s}
\newcommand{\mud}{\mu_{\hats}}
\newcommand{\F}{F}
\newcommand{\xn}{\tilde{\x}}
\newcommand{\Xn}{\tilde{\X}}
\newcommand{\tn}{\tilde{\theta}}
\newcommand{\tvn}{\tilde{\tv}}
\newcommand{\tminn}{\tilde{\theta}_{\min}}
\newcommand{\tmaxn}{\tilde{\theta}_{\max}}

\newcommand{\OMP}{\texttt{OMP}\xspace}
\newcommand{\OMPStep}{\texttt{OMP\_Step}\xspace}
\newcommand{\DS}{\texttt{D-OMP}\xspace}
\renewcommand{\DJ}{\texttt{DJ-OMP}\xspace}

\newcommand{\DC}{\texttt{DC-OMP}\xspace}

\newcommand{\DebL}{\texttt{Deb-Lasso}\xspace}
\newcommand{\DebLK}{\texttt{Deb-Lasso-K}\xspace}
\newcommand{\SIS}{\texttt{SIS}\xspace}
\newcommand{\SISSCAD}{\texttt{SIS-SCAD-K}\xspace}
\newcommand{\SISOMP}{\texttt{SIS-OMP-K}\xspace}
\newcommand{\SCAD}{\texttt{SCAD}\xspace}

\SetKwProg{Center}{At the fusion center}{}{}
\SetKwProg{MachinesM}{At each machine $m=1,\dots,M$}{}{}

\begin{document}

\title{Recovery Guarantees for Distributed-OMP\footnote{An earlier version of the paper was titled ``Distributed Sparse Linear Regression with Sublinear Communication''.}}
\author{{
		Chen Amiraz, Robert Krauthgamer and Boaz Nadler}\\
	Weizmann Institute of Science\\
}	

\maketitle

\begin{abstract}
 
We study distributed schemes for high-dimensional sparse linear regression, based on 
 orthogonal matching pursuit (\OMP). Such schemes are particularly suited for settings where a central fusion center is connected to end machines,  that have both computation and communication limitations. 	
	We prove that under suitable assumptions, distributed-\OMP schemes recover the support of 
	the regression vector 
	with communication per machine linear in its sparsity and logarithmic in the dimension. Remarkably, this holds even at low signal-to-noise-ratios, where individual machines are unable to detect the support. 
  Our simulations show that distributed-\OMP schemes 
  are competitive with 
  more 
  computationally intensive methods, and 
  in some cases even outperform them.   
\end{abstract}

\section{INTRODUCTION}
Sparse linear regression is a fundamental problem in machine learning, statistics and signal processing. 
Indeed, sparsity is a natural and widely applied modeling assumption in high dimensional settings. A sparsity assumption gives rise to the
variable selection problem, of 
identifying a small subset of variables which are most informative for a given prediction problem. 
We consider the popular sparse linear regression model with random noise, 
\begin{equation}\label{eq:reg_model}
y=\x^\top\tv+\sigma\xi,
\end{equation}
where $y\in\R$ is the response, $\x\in\mathbb{R}^d$ is a vector of explanatory variables, $\tv\in\mathbb{R}^d$ is the unknown vector of regression coefficients, 
$\xi\in\R$ is a standard normal random variable, i.e., $\xi\sim \Nor \left(0,1 \right) $, and $\sigma>0$ is the noise level.
We consider a high dimensional setting $d\gg 1$
with a vector $\tv$ 
of sparsity $ K=\norm{\tv}_0 \ll d$.
In a centralized setting, given $N$ samples
from \eqref{eq:reg_model}, common tasks are to 
accurately estimate $\tv$ as well as its support $\S=\mbox{supp}(\tv)=\{i\mid \theta_i \neq 0\}$. 
Many 
methods have been proposed and analyzed to solve these tasks, including combinatorial algorithms, linear programming, greedy approaches and regularization schemes \citep{mallat1993matching,tibshirani1996regression,chen2001atomic,miller2002subset,candes2005decoding,tropp2007signal,blumensath2008iterative,needell2009cosamp,dai2009subspace,bertsimas2016best,hastie2020best,amir2021trimmed}.

In various contemporary applications, the data is stored across multiple  machines. 
Moreover, due to communication or privacy constraints, the data at each machine cannot be sent to other machines in the network.  
Such cases 
bring about
various distributed learning problems, see  \cite{wimalajeewa2017application,jordan2019communication} and references therein.

A common distributed setting, which we also consider here, consists of $M$ machines connected in a star topology to a fusion center, with each machine having 
for simplicity an equal number of samples, $n=N/M$. 
For the sparse model \eqref{eq:reg_model}, 
some distributed methods attempt to recover the centralized solution that would have been computed by the fusion center, if it had access to all $N=nM$ samples 
of the $M$ machines.
Examples include optimization-based methods \citep{mateos2010distributed,ling2011decentralized,ling2012decentralized,fosson2016distributed,mota2011distributed},
Bayesian approaches \citep{makhzani2013distributed,khanna2016decentralized},
and greedy schemes \citep{sundman2012greedy,li2015decentralized,patterson2014distributed,han2015modified,chouvardas2015greedy}.
These methods are in general communication intensive, as they are iterative and may require many rounds to converge.
A simpler single round divide-and-conquer scheme, is for each machine to send its own estimate
of $\tv$ and for the fusion center to average these estimates. 
For a wide range of 
problems, the resulting estimator has a risk comparable to that of the centralized solution
\citep{rosenblatt2016optimality,wang2017efficient,jordan2019communication,liu2023communication}.
For
the sparse linear regression model \eqref{eq:reg_model}, 	
\cite{lee2017communication} and \cite{battey2018distributed} proposed a single round distributed debiased-Lasso scheme, and 
proved that under suitable conditions it achieves the same error rate as the centralized solution. Yet, these debiased-Lasso methods have two limitations: (i) the communication per machine is at least linear
in $d$;  and (ii) the computational costs are considerable, 
as
each machine has to solve $d+1$ Lasso problems. 
\citet{barghi2021distributed} and 
\citet{fonseca2023distributed} proposed 
debiased-Lasso methods with much less communication, 
where  
each machine sends to the center only the indices of its few largest
coordinates.

We consider distributed estimation of the sparse vector $\tv$ 
in the model \eqref{eq:reg_model}, 	under the following setting: 
The $M$ end machines have both limited processing power and a restricted communication budget. 
This 
is motivated by modern applications where end machines are computationally weak, but 
collect high dimensional data. For example, 
in spectrum sensing, 
a network of sensors 
continuously monitor and collect high dimensional data, and 
repeatedly need to estimate the current vector $\tv$. 
In these settings, computationally intensive methods such as debiased Lasso may be infeasible or prohibitively slow. In addition, regardless of computational considerations, most of the above methods are not applicable in high dimensions, as their
communication per machine is at least linear in $d$. 

As the quantity of interest $\tv$ is $K$-sparse with $K\ll d$, this 
gives rise to the following challenge: develop a scheme that accurately estimates the vector $\tv$ with number of operations per machine linear in $d$ and communication  \emph{sublinear} in $d$, and derive theoretical guarantees for it. 
Here we focus on accurately estimating the support of $\tv$. 
Indeed, as discussed in \citet{battey2018distributed,fonseca2023distributed}, 
given an accurate estimate of the support, an additional single round of communication allows distributed estimation of $\tv$ with the same error rate as in the centralized setting.

A natural base algorithm for machines with low computational resources is Orthogonal Matching Pursuit (\OMP), as it is one of the fastest methods for sparse recovery  \citep{chen1989orthogonal,pati1993orthogonal,mallat1993matching}. 
Several 
distributed-\OMP schemes to estimate the support of $\tv$, 
which are computationally fast and incur little communication, were proposed
in 
 \cite{duarte2005distributed,wimalajeewa2013cooperative,sundman2014distributed}.
In terms of theory, to the best of our knowledge, the only work to derive support recovery guarantees for distributed-\OMP methods is  \cite{wimalajeewa2014omp}. However, their analysis is restricted to a noise-less compressed-sensing setting, with samples $\x_i$ that are random and independent across machines.
Their proof is based on an underlying symmetry between
all non-support variables. 
In contrast, we consider a more general setting with deterministic samples corrupted by additive noise, for which their proof technique is not applicable.

Our key contribution is the derivation of a recovery guarantee for a distributed-\OMP scheme, see
Theorem \ref{thm:DJ-OMP}. Remarkably, our guarantee holds 
even at low signal-to-noise ratios (SNRs), where each individual machine fails to recover the support. 
The main 
challenge in our analysis is that 
the samples $\x_i$, assumed deterministic, may be similar (or even identical) across machines. 
Hence, at low signal-to-noise ratios,
several machines
might
send the \emph{same} incorrect support variable to the fusion center. Deriving a theoretical guarantee in this case requires a different and more delicate analysis than that of previous works.
Specifically, to bound the probability that a non-support variable is sent to the fusion center we use 
recent lower bounds on the maximum of correlated Gaussian random variables \citep{lopes2022sharp}.
Our analysis 
provides insight how distributed-\OMP methods can achieve exact support recovery even at low SNR where individual machines fail to do so.

To complement our theoretical analysis, 
in Section \ref{sec:sim}
we compare via simulations the support-recovery 
success of several algorithms including distributed-\OMP, 
debiased Lasso schemes \citep{lee2017communication,battey2018distributed,  barghi2021distributed}, 
and distributed variants of sure independence screening
(\SIS) \citep{fan2008sure}, which are also suitable for computationally weak machines.  
In a distributed variant of \SIS,  each machine first excludes variables weakly correlated to the response, and then estimates the sparse vector $\tv$ on the remaining ones via any appropriate algorithm.
In our experiments we considered
smoothly clipped absolute deviation (\SCAD) \citep{fan2001variable} and \OMP. 
As expected, our simulations show that the best performing scheme is debiased Lasso, but 
at the expense of significantly higher communication and computational costs. Interestingly, 
in comparison to a communication-restricted thresholded variant of debiased Lasso, distributed-\OMP methods perform comparably, and in 
some cases even outperform it, while being orders of magnitude faster. Furthermore, 
applying \SIS followed by \OMP at each machine
performed in some cases even slightly better that
distributed-\OMP. 
We conclude the paper with a discussion in Section \ref{sec:disc}.

\paragraph{Notation}
We use the standard $O(\cdot),\Omega(\cdot),\Theta(\cdot)$ notation to hide
constants independent of the problem parameters and $\tilde{O}(\cdot)$ to hide terms polylogarithmic in $d$. 
For functions $f,g$,  the notations $f=o(g)$ and $f\ll g$ mean that $f(d)/g(d)\to 0$ as
$d\to\infty$.
We say that an estimator $\hat S$ achieves exact support recovery with high probability if 
$\Pr\left[ \hat S=\indset\right] \to 1$ as both $d\to\infty$ and the number of
machines $M=M(d)\to\infty$ at a suitable rate. 
The smallest integer larger than or
equal to $x$ is denoted $\lceil x \rceil$.
For a standard Gaussian $Z\sim \Nor(0,1)$,  the complement of its
cumulative distribution function is $\Phi^c(t)=\Pr[Z>t ]$. 
We denote the inner product of two vectors ${\bm u},{\bm v}$ 
by $\langle {\bm u},{\bm v} \rangle= {\bm u}^\top {\bm v}$.

\section{PROBLEM SETUP}\label{sec:problem_setup}
We consider linear regression with a sparse coefficient vector in a distributed
setting, where \(M\) machines are connected in a star topology to a fusion center. 
Each machine $m\in[M]$ holds $n$ samples from the sparse regression model \eqref{eq:reg_model}, i.e.,
a design matrix $\X^{(m)}\in\R^{n\times d}$ and a response vector $\y^{(m)}\in \R^n$,
related via 
\begin{equation}\label{eq:sig_dec}
\y^{(m)} =  \X^{(m)}\tv +\sigma\err^{(m)},
\end{equation}
where 
$\err^{(m)}\sim \Nor\left(0,\bm{I}_n\right)$ and 
$\sigma$ is the unknown noise level. 
While the $M$ machines may have the same or similar design matrices,
their noises $\err^{(m)}$ are assumed to be independent.
We assume $\tv$ is \emph{$K$-sparse}, namely 
$\norm{\tv}_0=\abs{\supp(\tv)}=K$, with 
the value of $K$ known to the center.

The problem we consider is exact recovery of the support of $\tv$, which is a standard goal in sparse linear regression,  and has been widely studied in both non-distributed and distributed settings.  
We study this problem under the constraints that the $M$ machines have limited computational resources and limited communication with the fusion center. 
This setting is relevant in various applications including distributed compressed sensing and sensor networks.

\section{DISTRIBUTED OMP SCHEMES}\label{sec:algs}

\OMP-based schemes are popular for sparse support recovery, and are highly attractive in distributed settings where computation and communication are limited. 
We consider two distributed \OMP schemes to estimate the
support of $\tv$. 
Both schemes use the following subroutine, denoted 
\OMPStep, 
 which performs a single step of the \OMP algorithm, and outputs a new variable to be added to the current support set.
 As outlined in Algorithm \ref{alg:OMP_step}, 
  given a matrix $\X$, a vector $\y$, and a support set $\hatS$, the subroutine computes $\hat{\tv}$, the least squares approximation of $\tv$ on the support
$\hatS$ and 
its
residual vector $\res$. 
It then outputs an 
index $j\in[d]$ whose column
$\x_j$ has maximal correlation with  $\res$. 
A key property of \OMPStep is the orthogonality of the
residual to the columns of $\X$ in the set $S$.
Hence, the output of \OMPStep is a new index $j\notin \hatS$.

The simplest distributed \OMP method is for each machine to separately run \OMP for $K$ steps and send its $K$ locally-computed indices to the fusion center. The center 
estimates the support of $\tv$ by the 
$K$ indices that received the largest number of votes. 
To cope with low-SNR regimes where the top $K$ indices at individual machines
may not include all support indices, 
we propose a variant where each machine runs \OMP for a 
{\em larger} number of steps and thus sends a support of size $L>K$. This scheme,  which we call Distributed OMP (\DS), is outlined
in Algorithm \ref{alg:DS-OMP}.

\begin{algorithm2e}[!t]
	\caption{ \OMPStep}
	\label{alg:OMP_step}
	\DontPrintSemicolon
	\SetAlgoLined
	\SetKwInOut{Input}{input~}
	\SetKwInOut{Output}{output}

	\Input{~$\X\in\mathbb{R}^{n\times d}$, 
		$\y\in\mathbb{R}^n$, support set $\hatS$}
	\Output{support index $j $}
	compute
	 $\hat{\tv}= \mbox{argmin}_{\z\in\mathbb{R}^d, \supp(\z)=\hatS}
	\left\Vert
	\y-\X\z\right\Vert_2
	$ 
	\;
	compute residual $\res=\y-\X \hat{\tv}$ \;
	output index $j=\arg\max\left\lbrace \frac{  \lvert \langle \x
		_i,\res\rangle \rvert}{\norm{\x_i}}: i\in [d]\right\rbrace $ \;
	
\end{algorithm2e}

\begin{algorithm2e}[!t]
	\caption{\texttt{Distributed OMP (\DS)}}
	\label{alg:DS-OMP}
	\DontPrintSemicolon
	\SetAlgoLined
	\SetKwInOut{Input}{input}\SetKwInOut{Output}{output}
	
	\MachinesM{}{
		\Input{
			$\X^{(m)}\in\mathbb{R}^{n\times d}$, 
			$\y^{(m)}\in\mathbb{R}^n$, integer $L$ 
			}
		\Output{message $\hatS^{(m)}_L$ to center}
		initialize $\hatS^{(m)}_0=\emptyset$ \;
		\For{round $t=1,\dots,L$}{
			$j^{(m,t)}=\OMPStep\left( \X^{(m)},\y^{(m)},\hatS^{(m)}_{t-1}\right) $ \;
			update support set
			$\hatS^{(m)}_t=\hatS^{(m)}_{t-1}\cup \left\lbrace
			j^{(m,t)}\right\rbrace $
			
		}
		send  $\hatS^{(m)}_L$ to the center
	}	
	\Center{}{
		\Input{messages $\left\lbrace \hatS^{(m)}_L\right\rbrace _{m\in[M]}$, sparsity 
			$K$}
		\Output{estimated support $\hatS$ }
		for each index $j\in \left[ d\right] $, calculate the number of votes it received
		$\vv_j=\sum_{m\in [M]} \mathbb{1}\left\lbrace j \in \hatS^{(m)}_L \right\rbrace $ \;
		sort indices by number of
		votes, $\vv_{\pi(1)}  \geq \dots \geq \vv_{\pi(d)} $ \;
		return $K$ indices with most votes $\hatS=\left\lbrace \pi(1),\dots,
		\pi(K) \right\rbrace $	 \;
	}	
\end{algorithm2e}

\begin{algorithm2e}[!t]
	\caption{\texttt{Distributed Joint OMP}} 
	\label{alg:DJ-OMP}
	\DontPrintSemicolon
	\SetAlgoLined
	\SetKwInOut{Input}{input}\SetKwInOut{Output}{output}

	initialize $\hatS_0=\emptyset$
	\;
	
	\For{round $t=1,\dots,K$}{
		\MachinesM{}{
			$j^{(m,t)} = \OMPStep\left( \X^{(m)},\y^{(m)},\hatS_{t-1}\right)$ 

				send index $j^{(m,t)}$ to fusion center

		}	
		\Center{}{
			\Input{messages $j^{(m,t)}$, 
				sparsity 
				$K$}
			calculate number of votes for each index $j$, 
			$\vv^{(t)}_j=\sum_{m\in [M]} \mathbb{1}\left\lbrace j=j^{(m,t)} \right\rbrace $ \;
			find most voted index
			$j_t=\mbox{argmax}_{j} \vv^{(t)}_j$ 	 \;
			add $j_t$ to support set
			$\hatS_t=\hatS_{t-1}\cup\{j_t\}$ \;
			send $j_t$ to all machines \;
			if $t=K$ output $\hatS_K$
		}	
	}
\end{algorithm2e}

A second scheme, which we call Distributed Joint OMP (\DJ), computes the support set one index at a time, 
using $K$ communication rounds.
Starting with an empty support set $\hatS_0=\emptyset$, at each round $t=1,\dots,K$, the center
sends the current set $\hatS_{t-1}$ to the $M$ machines.
Then, each machine calls \OMPStep and sends the resulting index $j^{(m,t)} $ to the center.
At the end of each
round, the center adds to the support set an index $j_t$ that received the 
most
votes, $\hatS_t=\hatS_{t-1}\cup \{j_t\}$. 
After $K$ rounds, the center outputs the support set $\hatS_K$.
Since \OMPStep outputs an index not in the current set $\hatS_{t-1}$, at each
round $t$ of \DJ,  a new index is indeed added by the center, 
$j_t\notin\hatS_{t-1}$. 
This scheme is outlined in Algorithm \ref{alg:DJ-OMP}.	

\paragraph{Computation and Communication Complexity.}
Let us first analyze the number of operations in a single execution of \OMPStep. Given a support set $\hatS$, computing $\hat{\tv}$ via least squares involves multiplying a $\left| \hatS\right| \times n$ matrix by its transpose, and then inverting the resulting $\left| \hatS\right|\times\left| \hatS\right|$ matrix. Next, finding 
the index $j$ most correlated to the residual requires $d$ inner products of vectors in $\R^n$. For $\left| \hatS\right|$ sufficiently small, say $o\left(  d^{1/3}\right) $, the computational cost of \OMPStep is dominated by the latter step whose cost is $O(nd)$.

We now compare the two schemes \DJ and \DS with $L=K$. 
In terms of computational complexity, in both schemes each machine performs
the same number of operations. 
Thus, for $K=o\left(  d^{1/3}\right) $ their computational complexity per machine is $O(ndK)$. 
In terms of communication, in both schemes each machine sends and receives a total of $K$ indices,  
and so the communication per machine is $O(K\log d)$ bits. 
The main difference 
is that \DS performs a single round, whereas \DJ performs $K$ rounds. Hence, \DJ requires synchronization and is slower in comparison to \DS.

\paragraph{Related Works.}
Various distributed-\OMP methods were proposed in the past decade. 
\cite{wimalajeewa2013cooperative} considered the same \DS scheme as we do, with $L=K$.
In addition, they proposed a \DC algorithm, which is similar to \DJ. In \DC, at each round, instead of adding just one index to the support, 
the fusion center adds all
indices that received at least two votes. 
A distributed-\OMP approach for a different setting where each machine has its own regression vector $\tv^{(m)}$ was proposed in \cite{sundman2014distributed}. In their setting, the support sets of the $M$ vectors $\tv^{(m)}$ are assumed to be similar, and the $M$ machines are connected in a general topology without a fusion center.

\section{THEORETICAL RESULTS}\label{sec:theory}

Despite their simplicity, to the best of our knowledge, distributed-\OMP schemes lack rigorous mathematical support
and
only limited theoretical results have been derived for them. 
\citet{wimalajeewa2014omp} proved a support recovery guarantee for \DC, but only in 
a restricted noise-less compressed-sensing setting, where the entries of the design matrices are all random and i.i.d.\ across machines. 
In contrast, in this section we 
derive a support recovery guarantee
for \DJ, under  
a more general setting, where the design matrices are deterministic and potentially structured, and the responses $y$ are noisy.
Specifically, we prove in Theorem \ref{thm:DJ-OMP} that if the 
SNR is large enough (the non-zero entries
of $\tv$ are sufficiently large in absolute value), then with high probability
\DJ recovers the support set $\S$.
Remarkably, the SNR required by our theorem is well below that required for individual machines to succeed. 
Its proof appears in Appendix \ref{sec:proofs}.

Towards stating our result formally, we first recall recovery guarantees for \OMP on a single machine, and mathematically define the SNR in our problem.

\paragraph{Distributed Coherence Condition.}
The coherence of a matrix $\EM$ with columns 
$\Ev_j$ is defined as
\begin{equation}\label{eq:coh}
\mu\left( \EM\right) =\max_{i\neq
	j}\frac{\left|\innerp{\Ev_{i}}{\Ev_{j}}\right|}{\norm{\Ev_i}_2\norm{\Ev_j}_2}.
\end{equation}
A matrix $\EM$ satisfies the {\em Mutual Incoherence Property} (MIP) with respect to a
sparsity level $K$ if
\begin{equation}\label{eq:MIP}
\mu(\EM)<\frac{1}{2K-1}. 
\end{equation}
A fundamental result by \citet{tropp2004greed} is that in an ideal noise-less
setting ($\sigma=0$), the MIP condition \eqref{eq:MIP} is sufficient for exact support
recovery by \OMP.

In our distributed setting, each machine $m$ has its own design matrix $\X^{(m)}$ with coherence
$\mu^{(m)}=\mu(\X^{(m)})$.
We denote their maximal coherence by
\begin{equation}\label{eq:mumax}
\mumax= \mumax(\X^{(1)},\dots,\X^{(M)})=\max_{m\in [M]} \mu^{(m)}.
\end{equation}
We say that a set of matrices $\X^{(1)},\dots,\X^{(M)}$ satisfies the max-MIP
condition w.r.t.\ a sparsity level $K$ if
\begin{equation}
\label{eq:max_mip_cond}
\mumax<\frac{1}{2K-1}.  
\end{equation}
Eq. 
\eqref{eq:max_mip_cond}
implies that all machines satisfy the MIP condition \eqref{eq:MIP}. Hence, in a noise-less setting, 
\OMP at 
each machine will correctly recover the support of $\tv$.
The coherence plays a key role for \OMP recovery also in the presence of noise, as we discuss next.

\paragraph{SNR Regime.}
We formally define the SNR in our distributed setting. We then focus on an interesting regime, in which the SNR is sufficiently high for \OMP to recover the support of $\tv$ in a centralized setting,
where the center has access to all the samples from all machines, and yet too low for \OMP at a single machine to individually recover it. 
For a $K$-sparse vector $\tv\in\R^d$, a matrix $\EM\in\R^{n\times d}$ with coherence $\mu$ 
whose columns have unit norm, 
and a noise level $\sigma$,
define
\begin{equation}\label{eq:theta_crit}
\theta_{\crit}(\mu,d,K,\sigma)=\frac{\sigma\sqrt{2\log
		d}}{1-(2K-1)\mu}.
\end{equation}
Notice it is well defined under the MIP condition \eqref{eq:MIP}. 

As in previous works, to derive exact support recovery guarantees, we consider vectors $\tv$ whose non-zero entries 
have magnitude lower bounded by $\tmin$,
namely $\min_{k\in \S} \left|\theta_k\right| \geq \tmin $.
For a matrix $\EM$ with unit-norm columns, define the SNR as $r=\left( \frac{\tmin}{\theta_{\crit}(\mu,d,K,\sigma)}\right) ^2$.
Near the value $r=1$, \OMP (at a single machine) exhibits a phase transition from failure to success of support recovery. 
If the SNR is slightly higher, i.e.,
$r >\left( 1+\sqrt{\frac{\log K}{\log d}}\right)^2 $, then
with high probability
\OMP exactly recovers the support $\S$
\citep{ben2010coherence}. In contrast, if the SNR is slightly lower, i.e.,
$r<\left( 1-\sqrt{\frac{\log K}{\log d}}-\mu\right)^2 $,
then there are matrices $\EM\in\R^{n\times d}$ with coherence $\mu$ and $K$-sparse 
vectors
$\tv\in\R^d$ for which
given $\y=\EM\tv+\sigma\err$, 
\OMP fails 
with high probability
to recover the support
of $\tv$. In addition, this occurs empirically for several common families of matrices $\EM$ and
vectors $\tv$ \citep{amiraz2021tight}.

In our distributed setting the matrices $\X^{(m)}$ are assumed to be deterministic and 
do not necessarily have unit-norm columns. However, \eqref{eq:sig_dec} 
is equivalent to 
\begin{equation}\label{eq:norm_model}
\y^{(m)} =  \Xn^{(m)}\tvn^{(m)} +\sigma\err^{(m)},
\end{equation}	
where each column $\xn_j^{(m)}$ of the matrix $\Xn^{(m)}$ is scaled to have unit norm, i.e., 
$\xn_j^{(m)}={\x_j^{(m)}} / \|\x_j^{(m)}\|$, 
and accordingly $\tn_j^{(m)}=\|\x_j^{(m)}\|\theta_j$. Clearly, the support of each $\tvn^{(m)}$ is identical to that of $\tv$.
We assume that for a suitable $\tminn$, the vector $\tv$ satisfies that
\begin{equation}\label{eq:tminn}
	\min_m \| \x_k^{(m)}\|
\left|\theta_k\right| \geq 
		\tminn, 
\quad \forall k\in \S. 
\end{equation}

Given the above discussion, in our distributed setting we define
the SNR parameter $r$ as follows, 
\begin{equation}\label{eq:r_def}
r = \left(\frac{\tminn}{\theta_{\crit}( \mumax,d,K,\sigma)}\right)^2.
\end{equation}
If $r>1$ then $\tminn>\theta_{\crit}( \mu^{(m)},d,K,\sigma)$ at every machine $m\in[M]$, and hence \OMP in any single machine would recover the support of $\tv$ with high probability. 

Next, consider a centralized setting where all $N=Mn$ samples are available to the fusion center.
This setting corresponds to a response vector $\y\in\R^N$ and
measurement matrix $\X\in\R^{N\times d}$ formed by column-stacking
$\y^{(1)},\dots,\y^{(M)}$ and $\X^{(1)},\dots,\X^{(M)}$, respectively.
In analogy to \eqref{eq:r_def}, to guarantee
support recovery in this case, a sufficient condition 
is that the centralized SNR $r^{(c)}=
\left(
\frac{\tminn^{(c)} }
{\theta_{\crit}(\mu(\X),d,K,\sigma)}  
\right)^2 > 1$. Here
$\tminn^{(c)}$ is a value such 
that for all support indices $k\in\S$, 
$|\theta_k| \geq \tminn^{(c)}/\|\X_k\|$, where $\X_k$ is the $k$-th column of $\X$. 
Since $\norm{\X_k} \geq \sqrt{M} \min_m \norm{\x_k^{(m)}}$, then in a centralized setting
\OMP is guaranteed to succeed when  $\sqrt{M}\tminn>\theta_{\crit}(\mu(\X),d,K,\sigma)$.
Given the definition \eqref{eq:theta_crit} for $\theta_{\crit}$, 
an SNR regime that is interesting to study in the distributed setting is 
\begin{equation}\label{eq:reg_SNR_regime}
\frac{1}{M} \left( \frac{1-(2K-1) \mumax}{1-(2K-1)\mu(\X)}\right) ^2 < r 
< 1.
\end{equation}

As we show next, for a subrange of the SNR values in Eq. \eqref{eq:reg_SNR_regime}, 
the \DJ scheme can still achieve 
exact support recovery.

\subsection{Support Recovery Guarantee}

We present three assumptions for our recovery guarantee to hold. As \OMP is based on dot products between the residual and normalized columns of the design matrices,
we first introduce the following quantity that bounds how large these can be, 
\begin{equation}\label{eq:delta}
\delta=\delta\left(K,\mumax \right)  = \tfrac{\left( K-1\right)
	\mumax^{2}}{1-\left(K-2\right)\mumax}.
\end{equation}
As we show in Section \ref{app:lemmas_K}, under the max-MIP condition \eqref{eq:max_mip_cond}, $\delta\leq \mumax$.
Our first assumption is that the number of machines is sufficiently large, with the dependence on $K$ encoded in the quantity $\delta$. 
\begin{assumption}\label{assum:M}
$M\geq  \Mc\left(d,K,\mumax,r \right)$, where
\begin{equation}\label{eq:Mc_K-maintext}
	\Mc\left(d,K,\mumax,r\right)= K\left\lceil  \frac{16\log
		d}{\Phi^{c}\left(\tfrac{(1-\sqrt{r})\sqrt{2\log d}}{ \sqrt{1-\delta}(1-\mumax)
		}\right)}\right\rceil.
\end{equation}
\end{assumption}
In our analysis, we assume that $d\gg 1$ and that $\mumax$ is small. 
This implies that also $\delta$ is small and hence 
\begin{equation}\label{eq:Mcy_Ot}
\Mc\left(d,K,\mumax,r \right)\approx Kd^{\left(\frac{1-\sqrt{r}}{\sqrt{1-\delta}(1-\mumax)}\right)^{2}} ,
\end{equation}
{which follows from the approximation $\Phi^{c}(t)\approx e^{-t^2/2}$ and omitting $O(\log d)$ factors.}
Thus, larger SNR values (though still smaller than one), require fewer machines to guarantee support recovery.

To guarantee support recovery by \DJ, we also need to upper bound the probability that a non-support index is sent to the fusion center. 
As described in the appendix, 
for this we use a recent result on the left tail of the maximum of correlated Gaussian random variables \citep{lopes2022sharp}. 
The SNR that guarantees recovery thus depends on a parameter $\epsilon=\epsilon(K,\mumax)$, with smaller values of $\epsilon$ leading to a lower SNR. 
However, for our proof to work, $\epsilon$ cannot be arbitrarily small, and we set it as follows.  
\begin{assumption}\label{assum:eps}
The scalar $\epsilon=\epsilon(K,\mumax)$ satisfies 
\begin{equation}\label{eq:eps_cond_K}
\frac{\sqrt{\mumax+\delta}}{1+\sqrt{\mumax+\delta}}<\epsilon<1.
\end{equation}
\end{assumption}
Importantly, for $\mumax$ small,  $\epsilon$ can be chosen to be as small as {$O(\sqrt{\mumax})$. As detailed in the theorem below, this allows recovery at low SNRs.

Finally, we define a few quantities 
that characterize the lower bound we impose on the SNR $r$. 
Let
\begin{equation}\label{eq:Q_0_K}
Q_{0}\left(d,K\right)=\tfrac{\log\left(88\sqrt{2}K\right) }{\log d},
\end{equation}
and define $Q_{1}\left(d,K,\mumax,\epsilon\right)$
and $Q_2\left(d,K,\mumax\right)$ by 
\begin{equation}\label{eq:Q_1_K}
  Q_{1}
  =\tfrac{1-\left(1-\mumax\right)\sqrt{1-\delta}\left(\left(1-\epsilon\right)\sqrt{1-\mumax}-\sqrt{Q_{0}}\right)}{1-2\mumax K\sqrt{1-\delta}\frac{1-\mumax}{1-\left(2K-1\right)\mumax}},
\end{equation}
\begin{equation}\label{eq:Q_2_K}
  Q_2
  =\tfrac{\sqrt{2+2\left( \mumax+\delta\right) }\left(1+\sqrt{1-\delta}\left(1-\mumax\right)\sqrt{Q_{0}}\right)}{\sqrt{1-\delta}\left(1-\mumax\right)+\sqrt{2+2\left( \mumax+\delta\right)}}.
\end{equation}
\begin{assumption}[SNR Condition]\label{assum:r}
	The SNR $r$ is lower bounded as follows
	\begin{equation}\label{eq:r_cond_K}
		\sqrt{r} \geq
		\left\{
			\begin{array}{ll}
				Q_2 & (4K-1)\mumax-2K\mumax^2\geq 1 \\ 
				\min(Q_1,Q_2) & \mbox{\rm otherwise}
			\end{array}
				\right.
	\end{equation}
\end{assumption}

We can now state our support recovery guarantee.
The following theorem shows that under the above assumptions, 
the
\DJ algorithm, which indeed requires lightweight communication and computation, recovers the support of $\tv$, with high probability.

\begin{theorem}\label{thm:DJ-OMP}
	Under Assumptions \ref{assum:M}-\ref{assum:r} and the max-MIP condition \eqref{eq:max_mip_cond}, for sufficiently large $d=d(\epsilon)$, with probability at least 
	$1-2^{K+1}/d ,$ 
	\DJ with $K$ rounds recovers
	the support of the $K$-sparse vector $\tv$. 
\end{theorem}

Let us analyze the implications of the theorem when 
$K \ll d$ and $\mumax,\epsilon,\delta\ll 1$. In this case
$Q_1\approx\epsilon$ and $Q_2\approx \frac{\sqrt{2}}{1+\sqrt{2}}$. 
Hence, Assumption \ref{assum:r} is approximately
$r > (\min(Q_1,Q_2))^2 \approx \epsilon^2$ 
or $r\gtrsim \mumax$.  
Thus, there is a 
range of relatively low SNR values for which 
with a sufficiently large number of machines, \DJ is guaranteed to
recover the support,
even though individual machines fail to do so.
}

\begin{remark}
	Several works considered distributed settings where each machine has a different vector $\tv^{(m)}$, but they are all share the same support $\S$
	\citep{duarte2005distributed,ling2011decentralized,ling2012decentralized,wimalajeewa2014omp,li2015decentralized}. 
	Theorem \ref{thm:DJ-OMP} also holds in such cases, 
	under the following condition on the vectors
	$\tv^{(m)}$, instead of  \eqref{eq:tminn},
	$$
	\min_{m\in[M]} \norm{\x_k^{(m)}}\left|\theta_k^{(m)}\right|  \geq \tminn \quad \forall k\in \S.
	$$
\end{remark}

We now compare Theorem \ref{thm:DJ-OMP} to related works. 
\citet{amiraz2022distributed} studied 
distributed sparse mean estimation, which 
is
a special case 
of distributed sparse linear regression
where the design matrices are orthogonal.
They designed low-communication distributed schemes that provably recover the support
for a wide range of SNR values. 
However, their proofs rely on the design matrices being orthogonal, and 
do not generalize to incoherent matrices.
Their schemes are single-round, essentially using the orthogonality to recover all $K$ support indices in parallel,
in contrast to our \DJ scheme which has $K$ iterations, and requires a careful analysis of error propagation.
As mentioned above, \cite{wimalajeewa2014omp} considered a compressed-sensing setting with incoherent random matrices 
whose entries are drawn i.i.d.\ from the same distribution, and with no noise ($\sigma=0$). 
In both of these papers, a key property that greatly simplifies the analysis is that 
at all machines the probability for selecting a non-support index is the same
for all $k \notin \S$. 
Our theorem shows that even without this symmetry between the non-support indices, distributed-\OMP algorithms can achieve exact support recovery.

\section{SIMULATION RESULTS}\label{sec:sim}

We compare experimentally the following algorithms, which have different
computation and communication costs (see Table \ref{tab:alg}):
(i) \DebL
where each machine computes a 
debiased-Lasso estimate of $\tv\in\R^d$ and sends it to the center. The center averages these $M$
vectors and returns its top $K$ indices
\citep{lee2017communication,battey2018distributed}; 
(ii) \DebLK, a variant of \citet{barghi2021distributed}, where each machine sends the top $K$ indices of its debiased-Lasso estimate;
(iii) \SISSCAD, a distributed variant \SIS, where each machine performs variable screening followed by \SCAD \citep{fan2008sure}. It sends its resulting support set to the center, which selects the top $K$ indices by majority voting; 
(iv) \SISOMP, another distributed variant of \SIS. Here, 
each machine estimates its support set using \OMP on the remaining features;
(v) \DS with $L=K$; 
(vi) \DS with
$L=2K$; and 
(vii) \DJ. 
To illustrate the ability of \DJ to recover the support when individual machines fail, for
reference we also ran \OMP on a single machine, ignoring the data in all other $M-1$ machines.
Note that while \OMP-based schemes are essentially parameter free (beyond the sparsity $K$), the
debiased-Lasso schemes required knowledge of the noise level $\sigma$ in all machines. 

\begin{table}
	\caption {
		Communication and Computation Costs 
}\label{tab:alg}
	\begin{tabularx}{\columnwidth}{lXX}
		\textbf{Algorithm}  & \textbf{Communication cost} 	&  \textbf{Computational cost, $K\ll d^{1/3}$}\\
		\hline \\
		Single \OMP 		&  $\tilde{O}\left(K\right)$ 		&  $O\left(ndK\right)$\\
		\addlinespace
		\DebL				&  $\tilde{O}\left(d\right)$\footnotemark 	&  \multirow{2}{=}{\begin{tabular}[t]{@{}X@{}}solving $d+1$ Lasso optimization problems\end{tabular}}\\
		\DebLK				&  $\tilde{O}\left(K\right)$	&  \\
		& & \\
		\SISSCAD & {SNR dependent} & \multirow{2}{=}{\begin{tabular}[t]{@{}X@{}}$O\left(nd\right)$ \end{tabular}}\\
		\SISOMP & $\tilde{O}\left(K\right)$ &   \\ 
		\addlinespace
		\DS, $L\!=\!K$		
		&  \multirow{2}{=}{$\tilde{O}\left(K\right)$}	&  \multirow{2}{=}{\begin{tabular}[t]{@{}X@{}}$O\left(ndK\right)$ \end{tabular}}\\
		\DJ	
	\end{tabularx}
\end{table}
\footnotetext[1]{For \DebL, each machine sends the vector $\hat\theta^{(m)}$ itself, so the $\tilde{O}(\cdot)$ notation hides the number of bits used for each quantized value. }

\begin{figure}[!t]
	\centering
	\begin{subfigure}[t]{\columnwidth}
		\centering\includegraphics[width=0.75\columnwidth]{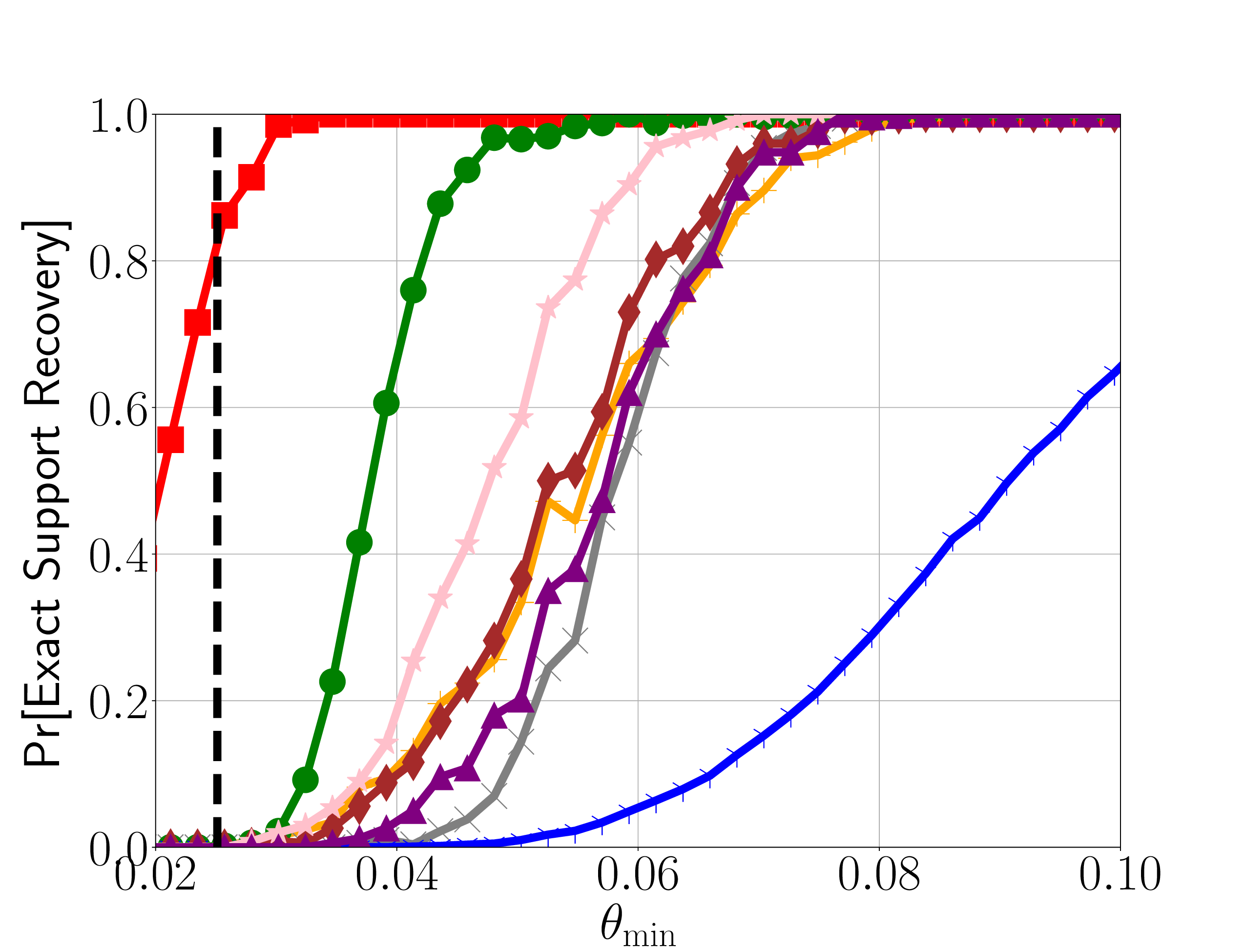}
		\caption{IID Design Matrices ($\alpha=0$)}
	\end{subfigure}
	\begin{subfigure}[t]{\columnwidth}
		\includegraphics[width=0.75\columnwidth]{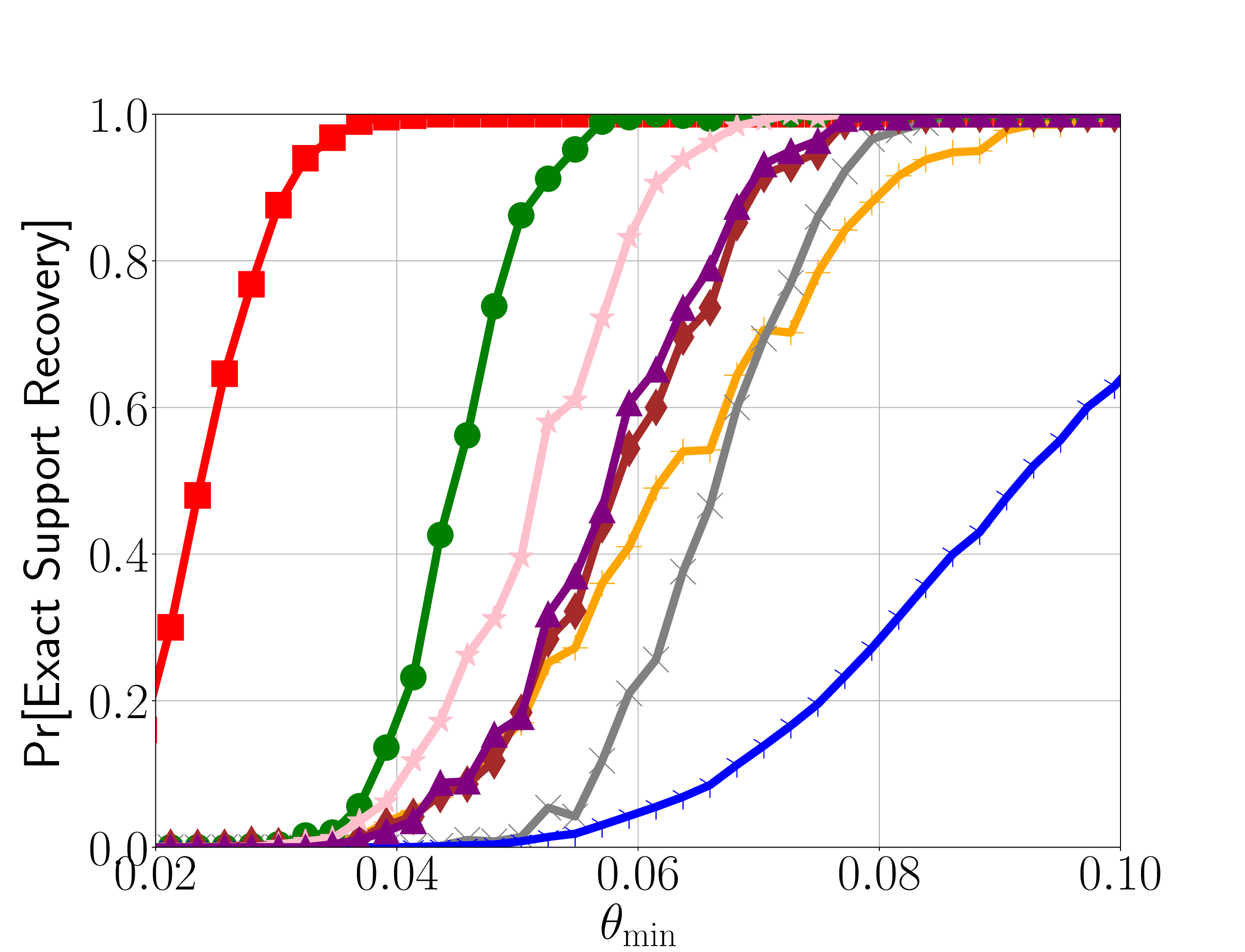}
		\centering\includegraphics[width=0.6\columnwidth]{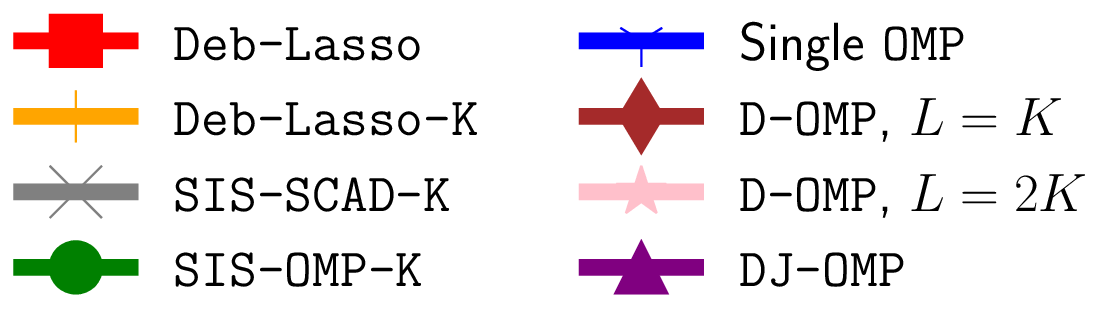}
		\caption{Correlated Design Matrices ($\alpha=0.1$)}
	\end{subfigure}
	\caption{Support Recovery as a Function of $\tmin$.
	}
	\label{fig:succ}
\end{figure}

\begin{figure}[!t]
	\centering
	\includegraphics[width=0.75\columnwidth]{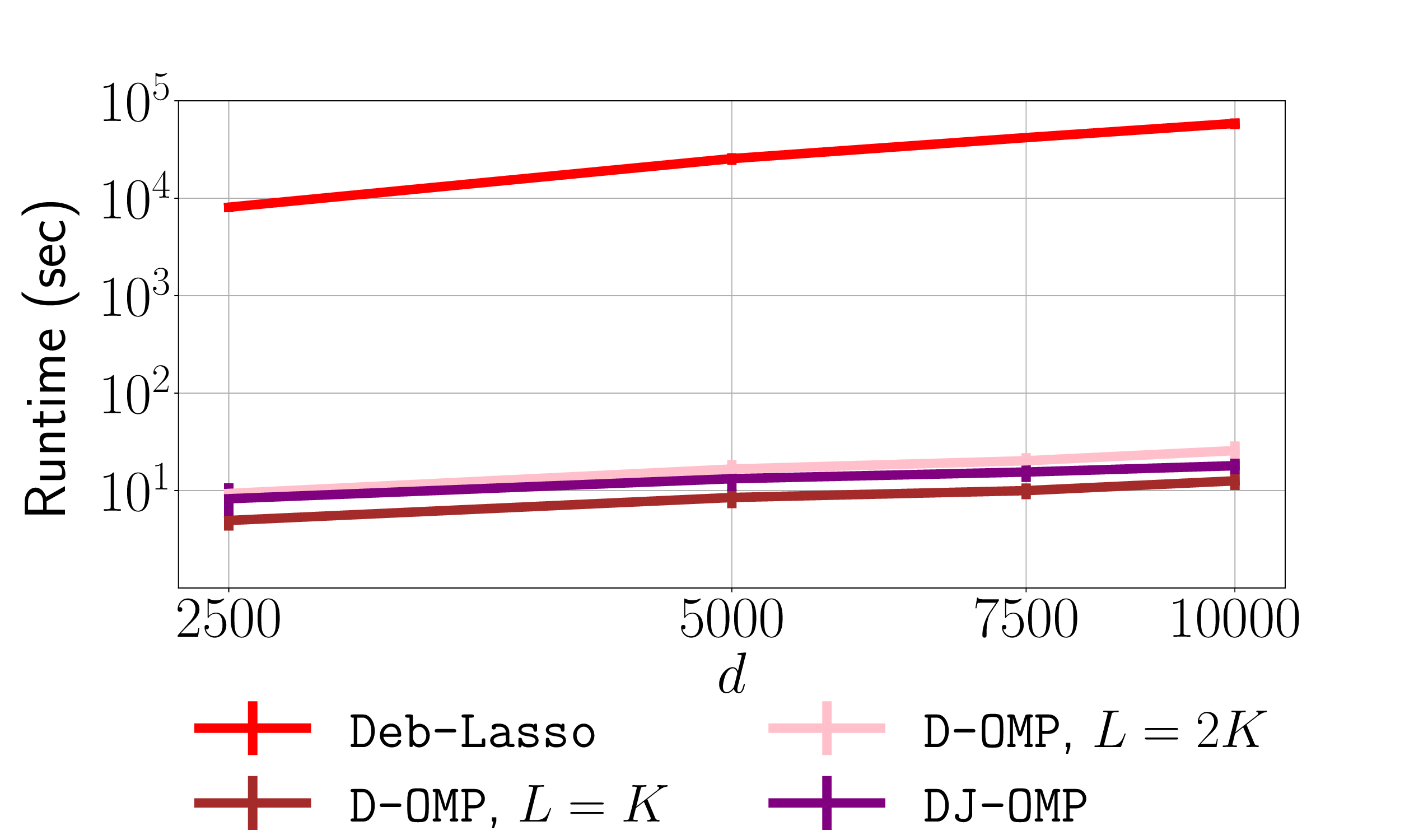}
	\caption{Runtime as a Function of $d$.}
	\label{fig:runtime}
\end{figure}

We now describe the experimental setup.
Each matrix $X^{(m)}$ is generated as follows. Each row is drawn
independently from $\Nor(\bm{0},\Sigma)$, where $\Sigma$ is a Toeplitz matrix
with $\Sigma_{ii}=1$ and $\Sigma_{ij}=\alpha^{\left| i-j\right| }$ for $i\neq j$ for some $\alpha\in [0,1)$.
In all settings, we generate $M=20$ such matrices, each containing $n=2000$ samples.
The noise level is $\sigma=1$, and the vector $\tv$ has a sparsity  $K=5$, with $\tv=\tmin \cdot
[1, -1.5, 2, -2.5, 3, 0, \dots, 0]^\top$.
The tuning parameter in the debiased-Lasso methods, which scales the $\ell_1$ term of each of the $d+1$ Lasso objectives, is set to $\lambda = 2\sigma \sqrt{\frac{\log d}{n}} $.
We consider two settings both of dimension $d=10000$. In Setting (a), $\alpha=0$, i.e., all matrix entries 
are i.i.d. $\mathcal N(0,1)$.
In Setting (b), $\alpha=0.1$, so the columns of $X^{(m)}$ are
weakly correlated.
Further implementation details appear in Appendix \ref{sec:implmnt}.

Figure \ref{fig:succ} illustrates the empirical success probability of the various algorithms as a function of $\tmin$ in the two settings outlined above.
Formally, for an algorithm $A$,
\[
p_{\text{ success }}^A(\tmin)=\frac{1}{J}\sum_{j=1}^{J} \mathbb{1}\left\lbrace \hatS_j^A(\tmin)=\S \right\rbrace, 
\] 
where $\hatS_j^A(\tmin)$ is the support set computed by algorithm $A$, for noise realization $j$ and lower bound $\tmin$ on the non-zero coefficients of $\tv$, and $J$ is the total number of noise realizations, set to $J=500$. 
The dashed vertical line in panel (a) is the lower bound $\theta_{\crit}( \mu(\X),d,K,\sigma)$ of Eq. \eqref{eq:theta_crit}, above which in a centralized setting, \OMP is guaranteed to recover the support. In panel (b), the MIP condition does not hold and the dashed line is not shown. Nonetheless, distributed schemes still succeed in this case.

Figure \ref{fig:succ} reveals several behaviors.
First, as anticipated, the distributed-\OMP algorithms perform inferior to \DebL, which incurs much higher computational and communication costs. 
Second, in accordance with Theorem \ref{thm:DJ-OMP}, the distributed-\OMP algorithms succeed at low SNR values, where \OMP on a single machine fails with high probability. 
Third, \DJ's performance is comparable to \DS with $L=K$. 
For scenarios requiring one-shot communication, \DS with more steps, $L=2K$ in this example, exceeds \DJ's performance, while incurring a slightly higher communication cost.
In Setting (a) where each entry of the matrices $\X^{(m)}$ is i.i.d.\ Gaussian, the performance of distributed-\OMP algorithms is on par with the computationally demanding \DebLK. Notably, in Setting (b) where the matrices $\X^{(m)}$ 
have correlated columns, distributed-\OMP methods surpass \DebLK.
In the context of variable screening methods, for a wide range of SNR values, a single machine often misses the full support set during the screening step. Yet, incorporating voting schemes enables distributed support recovery. 
Similar to \DebLK, \SISSCAD matches the performance of distributed-\OMP algorithms in Setting (a) but lags behind them in Setting (b). 
In all the studied settings, \SISOMP consistently outperforms both \SISSCAD and \DJ. 
A theoretical study of 
this behavior is an interesting topic for future research.

Figure \ref{fig:runtime} shows the runtime and error bars of several schemes, all implemented in Python, as a function of $d$ on a logarithmic scale. In this simulation, $\alpha=0$ and $\tmin=0.1$ and we averaged over $J=20$ realizations. 
The runtime of \DebLK is similar to that of \DebL, and thus not shown. 
As seen in the figure, distributed-\OMP methods are
more than three orders of magnitude faster than
\DebL. 
Not shown are the runtimes 
\SIS-based schemes, which are slower than \DJ, but are not directly comparable since for them we called from Python an existing R 
code. 

Finally, in Appendix \ref{app:add_exp} we show empirically 
that the number of machines to recover the support scales as $M\approx d^{\beta}$ for some $\beta < 1$,
in accordance with 
\eqref{eq:Mcy_Ot}.

\section{DISCUSSION}\label{sec:disc}
Distributed inference schemes can be assessed based on three criteria: (i) the SNR above which they succeed with high probability;  (ii) the communication per machine; and (iii) the computational requirements. 
Figure \ref{fig:Pareto} illustrates how distributed-\OMP schemes compare with other methods 
for sparse linear regression. 
Each scheme is represented by a circle, whose size encodes the computational cost per machine (on a logarithmic scale).
The $x$-axis is the SNR and the $y$-axis is the communication per machine. 

\begin{figure}[!t]
	\centering
\includegraphics[width=0.7\columnwidth]{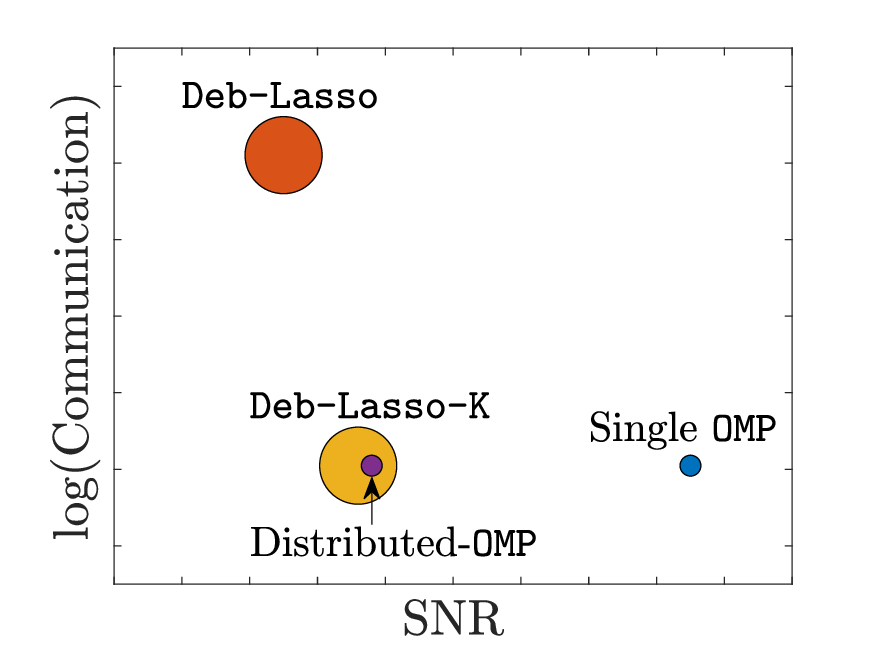}
	\caption{SNR, Communication and
		Computation Tradeoffs.
	}
	\label{fig:Pareto}
\end{figure}

As seen in the figure, known schemes exhibit a tradeoff between SNR and communication. 
When the SNR is sufficiently high for an individual machine to
recover the support of $\tv$, for example by \OMP, the fusion center may recover the support $\S$ by contacting
only one machine, incurring an incoming communication of only $O(K\log d)$ bits. 
Note that even in a noise-less setting, for the fusion center to recover the support, $K$ indices must be sent to the center, so $K\log d$ bits is a lower bound on the total required communication.
On the other hand, when the SNR is low, \DebL succeeds to recover the support of $\tv$ but
incurs a communication cost of $\tilde{O}\left( d\right) $ bits per machine, which might be prohibitive in  high-dimensional settings.

We conjecture that at low-SNR values, no distributed algorithm can achieve exact support recovery with 
communication per machine $O(K\log d)$ bits. 
We note that  
for closely related problems,
{achieving the centralized minimax $\ell_2$ risk or the
centralized prediction error is possible at low SNR
but requires communication cost of $\Omega(d)$ bits}
\citep{shamir2014fundamental,steinhardt2015minimax,acharya2019distributed}.
Our work shows that 
for a range of SNR values between these two extremes, 
distributed-\OMP algorithms do recover the support of $\tv$ with communication per machine $O(K\log d)$. 
An interesting open question is to determine the optimal rate at which the
required communication  decreases as a function of the SNR by any distributed algorithm that achieves exact support recovery.
Another interesting direction for future research is to 
characterize the tradeoff between communication costs and computational resources.

\bibliographystyle{plainnat}
\bibliography{sparse_bib}

\appendix
\section{PROOFS}\label{sec:proofs}
In this section we prove Theorem \ref{thm:DJ-OMP}.
For ease of presentation, in Section \ref{sec:K1} we state and prove Theorem \ref{thm:DJ-OMP1} which
addresses the simpler case $K=1$.
The proof of Theorem
\ref{thm:DJ-OMP} for the general case $K\geq 1$ appears in Section \ref{sec:proofK}.
The proofs of various auxiliary lemmas appear in Sections \ref{app:cher}-\ref{app:tech}.

Towards proving both theorems, we first present a few preliminaries, state useful lemmas and outline the proof.

\paragraph{Preliminaries.}
Recall that \DJ is an iterative algorithm, whereby at each round $t$, all $M$ machines call the subroutine
\OMPStep with the same input set $S_{t-1}$. In principle, except at the first round where $S_0=\emptyset$, this input
set depends on all the data in all $M$ machines. This statistical dependency significantly complicates the analysis. Instead, as discussed below, in our proof we will analyze a single round
of \DJ, assuming all machines are provided with a {\em fixed} input set $\hats$. 

Given an input set $s$ to the subroutine \OMPStep, each machine $m$ computes a sparse vector supported on $\hats$, i.e.,
\begin{equation}\label{eq:hat_tv}
\hat{\tv}^{(m)}= \arg\min_{\z\in\mathbb{R}^{d}}  \left\Vert
\y^{(m)}-\X^{(m)}\z\right\Vert_2 \text{ s.t. } \supp(\z)=\hats . 
\end{equation}
Then, it calculates the corresponding residual vector 
\begin{equation}\label{eq:res_m}
\res^{(m)}=\y^{(m)}-\X^{(m)} \hat{\tv}^{(m)}.
\end{equation}
Finally, each machine $m$ sends to the
fusion center the index
\begin{equation}\label{eq:m_picks_j}
j^{(m)}=\arg\max_{i\in [d]} \lvert \langle
\xn^{(m)}_i,\res^{(m)}\rangle \rvert ,
\end{equation}
where $\xn^{(m)}_i=\frac{\x^{(m)}_i}{\norm{\x^{(m)}_i}}$ is the $i$-th column of $\X^{(m)}$ divided by its norm.

As also described in  
Algorithm \ref{alg:DJ-OMP}, given the messages
sent by all 
$M$ machines, the fusion center 
computes a vector $\vv\in\mathbb{R}^d$, where $\vv_j$ counts the number of votes
received by index $j$ for all $j\in[d]$. As discussed in the main text, indices in $s$ receive no votes and at
each round a new index $\jc$ is chosen by the center, 
\[
\jc=\jc\left( s\right) =\arg\max_{j\in [d]\setminus \hats}  \vv_j.
\]
Towards proving that with high probability $\jc\in \S\setminus \hats$, 
we define an additional quantity
$\rho^{(m)} = \rho^{(m)}(s)$ that corresponds to the local SNR at machine $m$ given an input set $s$.
Denote 
\begin{equation}\label{eq:tmaxn}
\tmaxn^{(m)}= \tmaxn^{(m)}(s) =  \max_{k\in \S\setminus \hats} \left\lbrace  \norm{\x_{k}^{(m)}}\left| \theta_k\right| \right\rbrace.
\end{equation}
Similar to the definition of $r$ in Eq. \eqref{eq:r_def}, we define 
\begin{equation}\label{eq:rho}
\rho^{(m)} = \rho^{(m)}(s)  
= \left(\frac{\tmaxn^{(m)}}{\theta_{\crit}\left(\mumax,d,K,\sigma \right) } \right) ^2,
\end{equation}
where $\theta_{\crit}$ is defined in Eq. \eqref{eq:theta_crit}. 
Where clear from the context and to simplify notation we will not write the dependence on the input set $s$ explicitly. 
Note that by its definition, for any input set $s$ that is strictly contained in $\S$, it follows that $\rho^{(m)} \geq r$. 
As discussed in Section \ref{sec:theory}, if $\rho^{(m)}> 1 $, then 
with high probability 
machine $m$ would recover a support index, namely $j^{(m)} \in \S\setminus s$
\citep{amiraz2021tight}. 
Therefore, in what follows, we consider a worst case scenario whereby $\rho^{(m)}\leq 1$ in all machines $m\in[M]$.

\paragraph{Proof outline and lemmas.}
For simplicity we prove the theorem assuming the number of machines is the smallest that still satisfies
Assumption \ref{assum:M}, namely $M=\Mc\left(d,K,\mumax,r\right)$, with $M_c$ defined in Eq. \eqref{eq:Mc_K-maintext}. A larger number of machines would only increase the probability of exact support recovery.
The main idea of the proof is to show that at each of the $K$ rounds, with high probability the center
indeed chooses a support index. Specifically, consider 
a single round of \DJ with a fixed input set $\hats \subset \S$.
Then, for the center to choose an index $\jc\in \S\setminus\hats$, it
suffices that there exists some support index $k\in \S\setminus \hats$ that
received more votes than any non-support index, namely,
\begin{equation}\label{eq:supp_wins}
\vv_{k}>\max_{j\notin \S}\vv_{j}.
\end{equation}
A sufficient condition for \eqref{eq:supp_wins} to occur is that for some
suitable threshold $\tc=\tc(s)>0$, 
both
\begin{equation}\label{eq:sig_above}
\vv_k > \tc,
\end{equation}
and
\begin{equation}\label{eq:noise_below}
\max_{j\notin \S}\vv_{j}<\tc.
\end{equation}
As described below, our chosen threshold $\tc$ depends on the following quantity $F$, 
which provides a lower bound for the probability that a support index is sent to the center by one of the machines, 
\begin{equation}\label{eq:F}
\F\left(d,K,\mumax,r\right)=\frac{1}{2}\Phi^{c}\left(\tfrac{(1-\sqrt{r})\sqrt{2\log d}}{ \sqrt{1-\delta}(1-\mumax)
}\right).
\end{equation}
Note that by this definition, Eq. \eqref{eq:Mc_K-maintext} can be rewritten as
\begin{equation}\label{eq:Mc_K}
\Mc\left(d,K,\mumax,r\right)= K\left\lceil  \frac{8\log
	d}{\F\left(d,K,\mumax,r\right)}\right\rceil.
\end{equation}
We will show that Eqs. \eqref{eq:sig_above} and 
\eqref{eq:noise_below} indeed hold with high probability
with the following threshold
\begin{equation}\label{eq:tc}
\tc= \tc(s) =
\frac{\sum_{m\in [M] }\F\left(d,K,\mumax,\rho^{(m)}(s)\right)}{M\cdot \F\left(d,K,\mumax,r\right)}4\log d,
\end{equation}
where $r$, $\rho^{(m)}$ and $F$ are defined in Eqs. \eqref{eq:r_def}, \eqref{eq:rho}, and \eqref{eq:F} respectively.
Note that $\rho^{(1)},\dots,\rho^{(M)}$ and $\tc$, which all depend also on the subset $s$, are not assumed to be known to the center and are
only used in the proof.

The following Lemma \ref{lem:tc_bound} 
provides a lower bound for the threshold $\tc$, which
will be useful in our proofs. Its proof follows directly from the definition of
$\F$ in Eq. \eqref{eq:F} and appears in Section \ref{app:cher}.
\begin{lemma}\label{lem:tc_bound}
	Under the max-MIP condition \eqref{eq:max_mip_cond}, for any fixed $s\subset \S$,
	the threshold $\tc=\tc(s)$ defined in Eq. \eqref{eq:tc} satisfies
	\begin{equation}\label{eq:tc_bound}
	\tc\geq 4\log d.
	\end{equation}
\end{lemma}

The following Lemma \ref{lem:supp_prob} states that if the expected number of
votes for an index $k\in \S\setminus \hats$ is sufficiently high, then event
\eqref{eq:sig_above} occurs with high probability.
The next Lemma \ref{lem:non_supp_prob} shows that if the expected number of
votes for each non-support index $j\notin\indset$ is sufficiently low, then
event \eqref{eq:noise_below} occurs with high probability.
These lemmas follow from Chernoff bounds and are proved in
Section \ref{app:cher} as well.

\begin{lemma}\label{lem:supp_prob}
	Assume the max-MIP condition \eqref{eq:max_mip_cond} holds. 
	Fix $\hats\subset \S$, and let $\tc=\tc(s)$ be given by Eq. \eqref{eq:tc}.  
	If $\E\left[ \vv_k \right]\geq 2\tc$ for some $k\in \S\setminus\hats$, then 
	\[
	\Pr\left[ \vv_k \leq\tc  \right]  \leq \frac{1}{d}.
	\]
\end{lemma}

\begin{lemma}\label{lem:non_supp_prob}
	Assume the max-MIP condition \eqref{eq:max_mip_cond} holds. 
	Fix $\hats\subset \S$, and let $\tc=\tc(s)$ be given by Eq. \eqref{eq:tc}.  
	If for all non-support indices
	$j\notin\S$ it holds that 
	$\E\left[ \vv_j  \right]\leq\frac{\tc}{5}$ 
	then 
	\[
	\Pr\left[ \max_{j\notin \S} \vv_j \geq\tc  \right]  \leq
	\frac{1}{d}.
	\]
\end{lemma}

It remains to bound
$\E\left[ \vv_j  \right]$
from above for $j\in \S\setminus \hats$ and
from below for $j\notin \S$.
Towards this goal, denote by $p^{(m)}_{j}$ the probability that machine $m$ sends
index $j$, namely
\begin{equation}\label{eq:p_j_m}
p^{(m)}_{j}=\Pr\left[j^{(m)}=j \right],
\end{equation}
where $j^{(m)}$ is defined in Eq. \eqref{eq:m_picks_j}.

Since $\E\left[ \vv_j \right]=\sum_m p_j^{(m)}$, it suffices
to bound the probability $p_j^{(m)}$.
For ease of presentation, we first derive these bounds for the case $K=1$ in Section \ref{sec:K1}, and then extend them to the general case $K\geq 1$ in Section \ref{sec:proofK}.

\subsection{Support recovery guarantee for sparsity $K=1$}\label{sec:K1}
For completeness, we rewrite Assumptions \ref{assum:M}-\ref{assum:r} for this case.
Since $K=1$, by its definition in Eq. \eqref{eq:delta},  $\delta\left( 1,\mumax\right) =0$.
Hence, the quantity $\F$ simplifies to
\begin{equation}\label{eq:F_1}
\F\left(d,1,\mumax,r\right)=\frac{1}{2}\Phi^{c}\left(\frac{1-\sqrt{r}}{ 1-\mumax
}\sqrt{2\log d}\right),
\end{equation}
and the quantity $\Mc$ from Eq. \eqref{eq:Mc_K} reduces to 
\begin{equation}\label{eq:Mc1}
\Mc\left(d,1,\mumax,r \right) = \left\lceil  \frac{8\log
	d}{\F\left(d,1,\mumax,r\right)}\right\rceil .
\end{equation}
Thus, for $K=1$, Assumptions \ref{assum:M} and \ref{assum:eps} read as follows: 
\begin{assumption}\label{assum:M1}
	$M\geq  \Mc\left(d,1,\mumax,r \right)$.
\end{assumption}

\begin{assumption}\label{assum:eps1}
	The parameter $\epsilon=\epsilon(\mumax)$ satisfies 
	\begin{equation}\label{eq:eps_cond}
	\frac{\sqrt{\mumax}}{1+\sqrt{\mumax}}<\epsilon<1.
	\end{equation}
\end{assumption}

The quantity $Q_{0}$ reduces to
\begin{equation}\label{eq:Q_0}
Q_{0}\left(d,1\right)=\frac{\log\left(88\sqrt{2}\right)}{\log d}.
\end{equation}
In addition, the expressions for $Q_{1}$ and $Q_2$ simplify to
\begin{equation}\label{eq:Q_1}
Q_{1}\left(d,1,\mumax,\epsilon\right)=\frac{1-\left(1-\mumax\right)\left(\left(1-\epsilon\right)\sqrt{1-\mumax}-\sqrt{Q_{0}}\right)}{1-2\mumax
},
\end{equation}
\begin{equation}\label{eq:Q_2}
Q_2\left(d,1,\mumax
\right)=\frac{\sqrt{2+2\mumax}\left(1+\left(1-\mumax\right)\sqrt{Q_0}\right)}{1-\mumax+\sqrt{2+2\mumax}}.
\end{equation}
Finally, for $K=1$, Assumption \ref{assum:r} on the SNR is: 
\begin{assumption}[SNR Condition]\label{assum:r1}
	The SNR is sufficiently high, 
	\begin{equation}\label{eq:r_cond}
	\sqrt{r} \geq \left\{
	\begin{array}{ll}
	Q_2 & \mumax \geq 1/2 \\
	\min(Q_1,Q_2) & \mbox{otherwise}
	\end{array}
	\right.
	\end{equation}
\end{assumption}

\begin{theorem}\label{thm:DJ-OMP1}
	Under Assumptions \ref{assum:M1}-\ref{assum:r1} and the max-MIP condition $\mumax<1$,
	for sufficiently large $d=d(\epsilon)$, with probability at least 
	$1-2/d ,$ a single round of 
	\DJ recovers the support of a $1$-sparse vector $\tv$.
\end{theorem}
A few remarks are in place.
First, note that when $K=1$, \DS with $L=1$ reduces to the same algorithm as \DJ, and thus this result holds for this algorithm as well.
Second, as mentioned in Section \ref{sec:theory}, when $\mumax\ll 1$ condition \eqref{eq:r_cond} roughly translates to
$r\gtrsim \epsilon^2$, and hence $r\gtrsim \mumax$.  
Thus, there is a range of relatively low SNR values for which \DJ succeeds to
recover the support, even though the probability of any single machine to do so is very low.

\subsubsection{Proof of Theorem \ref{thm:DJ-OMP1}}\label{sec:proof1}
When $K=1$, 
only a single round is performed with an input set 
$\hats=\emptyset$. Thus it trivially holds that $\hats\subset \S$.
In addition, \OMPStep simplifies to the following
procedure. 
At each contacted machine $m$, the residual is simply the response vector, i.e.,
$\res^{(m)}=\y^{(m)}$.
Thus, the index sent by machine $m$ to the fusion center is given by 
\begin{equation}\label{eq:m_picks_j1}
j^{(m)}=\arg\max_{i\in [d]}   \lvert \langle
\xn^{(m)}_i,\y^{(m)}\rangle \rvert.
\end{equation}

Another simplification in the case $K=1$ is that the support set contains only one index, which we denote by $k$, i.e., $\S=\left\lbrace k\right\rbrace $.	
To prove Theorem \ref{thm:DJ-OMP1}, we derive a lower bound on the probability $p_k^{(m)}$ for the support index $k$ in the following Lemma \ref{lem:pmax1} and 
an upper bound on the probability $p_j^{(m)}$ for each non-support index $j\notin \S$ in the following Lemma \ref{lem:p_j_t1}.
Their proofs appear in Section \ref{app:lemmas_1}
and are based on a probabilistic analysis of the inner products
between the response vector $\y^{(m)}$, which consists of signal and noise,
and different columns $\xn_{i}$.

\begin{lemma}\label{lem:pmax1}
	Assume that $\left\| \tv\right\| _0 = K =1$ and let $\S=\left\lbrace k\right\rbrace =\supp\left\lbrace \tv\right\rbrace $.
	Further assume that the max-MIP condition \eqref{eq:max_mip_cond} holds.
	For sufficiently large $d$,
	for each machine $m$, 
	\begin{equation}\label{eq:pmax1}
	p_{k}^{(m)} \geq \F\left(d,1,\mumax,\rho^{(m)} \right),
	\end{equation}
	where $p_{k}^{(m)}$  and $\F$ are defined in Eqs. \eqref{eq:p_j_m} and \eqref{eq:F_1} respectively.
\end{lemma}

\begin{lemma}\label{lem:p_j_t1}
	Assume that $\left\| \tv\right\| _0 = K =1$ and let $\S=\supp\left\lbrace \tv\right\rbrace $.
	Further assume that $\rho^{(m)}$ of Eq. \eqref{eq:rho} satisfies $\rho^{(m)}\leq 1$ for each machine $m$ and that the max-MIP condition \eqref{eq:max_mip_cond} holds.
	If $\epsilon$ satisfies Assumption \ref{assum:eps1}, the SNR parameter $r$ satisfies Assumption \ref{assum:r1}, and the dimension $d=d(\epsilon)$ is sufficiently large, then
	for each machine $m$ and
	each non-support index $j\notin \S$,
	\begin{equation}\label{eq:p_j_t_m_bound1}
	p_{j}^{(m)}   \leq\frac{\F\left(d,1,\mumax,\rho^{(m)} \right)}{11}.
	\end{equation}
\end{lemma}

We now formally prove Theorem \ref{thm:DJ-OMP1} by combining the above
lemmas.

\begin{proof}[Proof of Theorem \ref{thm:DJ-OMP1}]
	For simplicity, we assume that the number of machines is $M=\Mc\left(d,1,\mumax,r \right)$, since a larger number of machines would only increase the probability of successful support recovery.
	We first analyze the probability that event \eqref{eq:sig_above} occurs.
	By Lemma \ref{lem:pmax1}, for the support index $k\in \S$, its expected number of votes is $\E\left[ \vv_k 
	\right]=\sum_{m\in [M]} p_{k}^{(m)}\geq \sum_{m\in [M] }\F\left(d,1,\mumax,\rho^{(m)} \right)$. 
	By the definitions of $\tc$ in Eq.
	\eqref{eq:tc} and $\Mc$ in Eq. \eqref{eq:Mc1},
	\[
	\E\left[\vv_{k}\right]\geq\frac{\Mc\cdot \F\left(d,1,\mumax,r \right)}{4\log d}\cdot\tc=\left\lceil \frac{8\log d}{\F\left(d,1,\mumax,r\right)}\right\rceil \frac{\F\left(d,1,\mumax,r\right)}{4\log d}\cdot\tc\geq2\tc.
	\]
	By Lemma \ref{lem:supp_prob}, the event \eqref{eq:sig_above} occurs with
	probability at least $1-1/d$.
	
	Next, we analyze the probability that event \eqref{eq:noise_below} occurs.
	Fix a non-support index $j\notin \S$.
	Since $\rho^{(m)} \leq 1$, then by Lemma \ref{lem:p_j_t1}, its expected number of votes is $\E\left[ \vv_j  \right]=\sum_{m\in [M]} p_{j}^{(m)}\leq\frac{1}{11}\sum_{m\in[M]}\F\left(d,1,\mumax,\rho^{(m)}\right)$.
	By the definitions of $\tc$ in Eq. \eqref{eq:tc} and $\Mc$ in Eq.
	\eqref{eq:Mc1},
	\[
	\E\left[\vv_{j}\right]\leq\frac{1}{11}\left\lceil \frac{8\log d}{\F\left(d,1,\mumax,r\right)}\right\rceil \frac{\F\left(d,1,\mumax,r\right)}{4\log d}\tc<\frac{\tc}{5}.
	\]
	The last inequality is justified as follows. 
	Recall that $\lceil  x\rceil\leq x+1$ for all $x$. 
	Thus, 
	$$
	\left\lceil \frac{8\log d}{\F\left(d,1,\mumax,r\right)}\right\rceil \frac{\F\left(d,1,\mumax,r\right)}{4\log d}\leq2+\frac{\F\left(d,1,\mumax,r\right)}{4\log d}.
	$$ 
	By the definition of $F$ in Eq. \eqref{eq:F_1}, it follows that 	$\F\left(d,1,\mumax,r\right)\leq 1$. Hence, when $d\geq 8$, then $\log d >2$, and the term $\frac{\F\left(d,1,\mumax,r\right)}{4\log d}\leq\frac{1}{8}$.
	Hence, 
	by Lemma \ref{lem:non_supp_prob}, the event \eqref{eq:noise_below} occurs with
	probability at least $1-1/d$. A union bound completes the proof. 
\end{proof}

\subsection{Proof of Theorem \ref{thm:DJ-OMP}}\label{sec:proofK}
We now prove that with high probability, 
\DJ succeeds to recover the support of $\tv$ with general sparsity level $K$. 
The proof relies on the following lemma, which bounds the probability that, given a fixed input set $s$, the
center chooses an incorrect index at a single round of the algorithm.
\begin{lemma}\label{lem:cent}
	Let $\hats\subset[d]$ be a fixed set of indices given as input to a single round of
	\DJ and denote by $\jc\left( s\right) $ the index chosen by the center at the end of this
	round.
	Under Assumptions \ref{assum:M}-\ref{assum:r} and the max-MIP condition \eqref{eq:max_mip_cond}, for sufficiently large $d=d(\epsilon)$, if $\hats\subset \S$ then the index $\jc\left( s\right) $ also
	belongs to the support set $\S$ with high probability. Specifically, 
	\begin{equation}\label{eq:fail_prob}
	\Pr \left[ \jc\left( \hats\right) \notin \S \right]\leq  2d^{-1}.
	\end{equation}
\end{lemma}
First, let us show how Theorem \ref{thm:DJ-OMP} follows directly from Lemma
\ref{lem:cent}.
\begin{proof}[Proof of Theorem \ref{thm:DJ-OMP}]
	Recall that \DJ starts with $\hatS_{0}=\emptyset $, adds exactly one new
	index to the estimated support set at each round, and runs for exactly $K$
	rounds. We denote by $\hatS_1,\hatS_2,\ldots,\hatS_K$ the index sets found by the center after $t=1,2,\ldots,K$ distributed rounds of \DJ, respectively.
	
	Our goal is to upper bound the probability 
	that $\hatS_K$, the output of \DJ after $K$ rounds, is not the true support set $\S$. 
	To this end we decompose this failure probability according to the round at which the failure occurred,
	\[
	\Pr\left[\hatS_{K}\neq\S\right]=\sum_{t=1}^{K}\sum_{\overset{s_{t-1}\subset\S}{|s_{t-1}|=t-1}}\Pr\left[j_{t}\left(s_{t-1}\right)\notin\S \text{ and }\hatS_{t-1}=s_{t-1}\right].
	\]
	Directly analyzing each of the terms above is challenging due to the statistical dependency between the set of indices found so far $S_{t-1}$, and the new index found in the current round. To overcome this, we use the inequality $\Pr[A\cap B]\leq \Pr[A]$, which gives
	\begin{equation*}
	\Pr\left[\hatS_{K}\neq\S\right]\leq\sum_{t=1}^{K}\sum_{\overset{s_{t-1}\subset\S}{|s_{t-1}|=t-1}}\Pr\left[j_{t}(s_{t-1})\notin\S\right].
	\end{equation*}	
	Since now the set $s_{t-1}$ is fixed, we can bound each term via
	Lemma \ref{lem:cent}. This gives 
	\[
	\Pr\left[\hatS_{K}\neq\S\right]\leq\frac{2}{d}\sum_{t=1}^{K}{K \choose t-1}=\frac{2}{d}\left(2^{K}-1 \right) \leq\frac{2^{K+1}}{d},
	\]
	which completes the proof.
\end{proof}
Next, we prove Lemma \ref{lem:cent}. 
Since $\hats \subset \S$, we need to bound the probability $p_j^{(m)}$ of Eq. \eqref{eq:p_j_m} for $j\in \S\setminus\hats$ and for $j\notin \S$.
We shall do so using the following two lemmas.	
The first one, Lemma \ref{lem:pmax}, lower bounds a different quantity
$q^{(m)}$ defined as the probability that the index sent by machine $m$
belongs to the support $\S\setminus \hats$,
\begin{equation}\label{eq:qm_def}
q^{(m)}=q^{(m)}(s)  =\Pr\left[j^{(m)}\in \S\setminus \hats \right].
\end{equation}
Lemma \ref{lem:p_j_t} upper bounds $p_{j}^{(m)}$ for each $j\notin \S$.
Their proofs appear in Section \ref{app:lemmas_K}.

\begin{lemma}\label{lem:pmax}
	Assume that the max-MIP condition \eqref{eq:max_mip_cond} holds.
	For each machine $m$, 
	for sufficiently large $d$,
	\begin{equation}\label{eq:pmax}
	q^{(m)} \geq \F\left(d,K,\mumax,\rho^{(m)} \right) ,
	\end{equation}
	where $q^{(m)}$  and $\F$ are defined in Eqs. \eqref{eq:qm_def} and \eqref{eq:F} respectively.
\end{lemma}
\begin{lemma}\label{lem:p_j_t}
	Assume that $\rho^{(m)}$ of Eq. \eqref{eq:rho} satisfies $\rho^{(m)}\leq 1$ for each machine $m$ and that the max-MIP condition \eqref{eq:max_mip_cond} holds.
	If $\epsilon$ satisfies Assumption \ref{assum:eps}, the SNR parameter $r$ satisfies Assumption \ref{assum:r}, and the dimension $d=d(\epsilon)$ is sufficiently large, then
	for each machine $m$ and
	each non-support index $j\notin \S$,
	\begin{equation}\label{eq:p_j_t_m_bound}
	p_{j}^{(m)}   \leq\frac{\F\left(d,K,\mumax,\rho^{(m)} \right)}{11K}, 
	\end{equation}
	where $p_{j}^{(m)}$  and $\F$ are defined in Eqs. \eqref{eq:p_j_m} and \eqref{eq:F} respectively.
\end{lemma}

We now formally prove Lemma \ref{lem:cent} by combining the above
lemmas.
\begin{proof}[Proof of Lemma \ref{lem:cent}]
	As mentioned above, for simplicity, we prove the lemma assuming that the number of machines is $M=\Mc\left(d,K,\mumax,r\right)$, since a larger number of machines would only increase the probability of exact support recovery.
	We first analyze the probability that event \eqref{eq:sig_above} occurs.
	Since $\hats\subset \S$, the set of support indices not yet found is  $\S\setminus \hats$. Let $\vv( \S\setminus \hats)=\sum_{k\in \S\setminus \hats}\vv_k$ be the total number of votes received for all these support indices combined.
	By Lemma \ref{lem:pmax}, the expected number of votes is $\E\left[ \vv( \S\setminus \hats) \right] = \sum_{m\in [M]} q^{(m)}\geq \sum_{m \in [M]} F\left(d,K,\mumax,\rho^{(m)} \right)$.
	By definition of $\tc$ in Eq. \eqref{eq:tc},
	\[
	\E\left[\vv(\S\setminus\hats)\right]\geq\frac{\Mc\left(d,K,\mumax,r\right)\F\left(d,K,\mumax,r\right)}{4\log d}\cdot \tc.
	\]
	By definition of $\Mc$ in Eq. \eqref{eq:Mc_K},
	\[
	\E\left[\vv(\S\setminus\hats)\right]\geq K\left\lceil \frac{8\log d}{\F\left(d,K,\mumax,r\right)}\right\rceil \frac{F\left(d,K,\mumax,r\right)}{4\log d}\cdot\tc\geq2K\tc.
	\]
	By an averaging argument, there exists a support index $k\in \S\setminus \hats$
	for which $\E[\vv_k] \geq \frac{1}{|\S\setminus\hats|} \E\left[\vv(\S\setminus\hats)\right] \geq 2\tc$. 
	Thus, by Lemma \ref{lem:supp_prob}, the event \eqref{eq:sig_above} occurs with
	probability at least $1-1/d$.

	Similarly to the proof of Theorem \ref{thm:DJ-OMP1}, Lemmas \ref{lem:non_supp_prob} and \ref{lem:p_j_t} imply that the event \eqref{eq:noise_below} also occurs with
	probability at least $1-1/d$.
	The only change in the proof is that $M_c$ now has a factor of $K$,
	which cancels with the $1/K$ factor in Lemma \ref{lem:p_j_t}.
	A union bound completes the proof.
\end{proof}

\subsection{Proofs of Lemmas  \ref{lem:tc_bound}, \ref{lem:supp_prob} and \ref{lem:non_supp_prob}}\label{app:cher}
We first prove Lemma \ref{lem:tc_bound} and then use it to prove Lemmas
\ref{lem:supp_prob} and \ref{lem:non_supp_prob}.
\begin{proof}[Proof of Lemma \ref{lem:tc_bound}]
	By its definition in Eq. \eqref{eq:F}, the function $F$ is monotonic increasing in its fourth argument. Next, by Eq. \eqref{eq:tmaxn}, $\tminn\leq\tmaxn^{(m)}$, and thus $r\leq\rho^{(m)}$ for each $m\in[M] $. Hence, 
	\[
	\frac1M \sum_{m\in [M]} 
	\frac{ 
		F(d,K,\mumax,\rho^{(m)} ) }
	{
		F(d,K,\mumax,r) 
	} \geq 1
	\]
	Inserting this inequality into the definition of $\tc$, in Eq. \eqref{eq:tc} concludes
	the proof. 
\end{proof}

In the proofs below we use the following Chernoff bounds.
\begin{lemma}[\cite{chernoff1952measure}]\label{lem:Chernoff}
	Suppose $X_1,\dots,X_d$ are independent Bernoulli random variables and let $X$
	denote their sum. Then, for any $\phi\geq 0$,
	\begin{equation}\label{eq:ChenoffAbove}
	\Pr\left[X\geq\left(1+\phi\right)\E\left[X\right]\right]\leq
	e^{-\frac{\phi^{2}\E\left[X\right]}{2+\phi}}, 
	\end{equation}  
	and for any $0\leq \phi\leq 1$,
	\begin{equation}\label{eq:ChenoffBellow}
	\Pr\left[X\leq\left(1-\phi\right)\E\left[X\right]\right]\leq
	e^{-\frac{\phi^{2}\E\left[X\right]}{2}}. 
	\end{equation} 
\end{lemma}

Next, we  introduce a few notations.
Denote the indicator that machine $m$ sends index $k$ by
$I^{(m)}_{k}=\mathbb{1}\left\lbrace j^{(m)}=k \right\rbrace $.
The number of votes that $k$ receives is thus $\vv_k=\sum_{m\in[M]} I^{(m)}_{k} $.
Further denote $E_k=\E[\vv_k ]$. 
Recall that the noises $\left\lbrace \err^{(m)}\right\rbrace_{m\in[M]} $ are independent.
Hence, for a fixed $\hats$, the indicators $\left\lbrace I^{(m)}_{k}\right\rbrace_{m\in[M]} $ are independent of each other.
We now combine Lemmas \ref{lem:Chernoff} and \ref{lem:tc_bound} to prove Lemmas
\ref{lem:supp_prob} and \ref{lem:non_supp_prob}.

\begin{proof}[Proof of Lemma \ref{lem:supp_prob}]
	By the discussion above, we may apply the Chernoff bound \eqref{eq:ChenoffBellow} to the sum $\vv_k$. Using the assumption $E_k\geq 2\tc$ and Lemma \ref{lem:tc_bound}, we obtain
	\begin{equation*}\label{eq:B_bound1}
	\Pr\left[ \vv_k < \tc \right]
	\leq \Pr\left[\vv_k < \tfrac12 E_k \right]
	\leq \exp\left(-E_k/8\right)
	\leq \exp\left(-\tc/4\right) \leq 1/d . 
	\end{equation*}
\end{proof}

\begin{proof}[Proof of Lemma \ref{lem:non_supp_prob}]
	Fix $j\notin \S$ and let $
	\phi_j  =  \frac{\tc}{E_j} - 1.$
	The probability of interest is monotonically increasing in $E_j$. Hence, 
	it suffices to prove the lemma for $E_j = t_c/5$. In this case
	$\phi_j = 4$, and $\phi_j/(2+\phi_j) = 2/3$.
	Applying the Chernoff bound \eqref{eq:ChenoffAbove} to the sum $\vv_j$, we obtain
	\begin{equation*}
	\Pr\left[ \vv_{j}>\tc\right]
	=\Pr\left[\vv_{j}>(1+\phi_j)E_{j} \right]
	\leq\exp\left(-\frac{\phi_j^2}{2+\phi_j} E_{j}\right)
	\leq\exp\left(-\frac{8\tc}{15}\right).
	\label{eq:v_j_M4}
	\end{equation*}
	By Lemma \ref{lem:tc_bound}, the above probability is smaller than $d^{-2}$, and by applying a union bound we conclude that
	\[
	\Pr\left[ \max_{j\not\in\indset} \vv_j > \tc \right]  \leq
	(d-K)\Pr\left[ \vv_j > \tc \right]  \leq 1/d.
	\]
\end{proof}

\subsection{Proofs of Lemmas  \ref{lem:pmax1} and \ref{lem:p_j_t1}}\label{app:lemmas_1}
We begin with a few definitions and notations.
For a  set of indices $\cal{I}$, let
$\bm{u}_{\mid\cal{I}}\in \R^{|\cal{I}|}$ be the restriction of the vector
$\bm{u}$ to $\cal{I}$.
Similarly, for a matrix $\EM\in \R^{n\times d}$, let $\EM_{\mid\cal{I}}\in
\R^{n\times |\cal{I}|}$ be the restriction of the matrix $\EM$ to the columns indexed by 
$\cal{I}$.
Further denote by $\EM^{\dagger}$ the Moore-Penrose pseudo inverse of the matrix
$\EM$, i.e., $\EM^{\dagger}=\left(\EM^\top \EM \right)^{-1}\EM^\top $ and notice that
$\EM^{\dagger}\EM = \bm{I}  $. 
Lastly, recall that $\Xn^{(m)}\in \R^{n\times d}$ is the column-normalized matrix
in machine $m$ and denote by $\P^{(m)}_{\mathcal{I}}\in \R^{n\times n}$ an orthogonal
projection onto the span of $\Xn^{(m)}_{\mid\mathcal{I}}$, i.e.,
\begin{equation}\label{eq:P}
\P^{(m)}_{\mathcal{I}}=\Xn^{(m)}_{\mid\mathcal{I}}\left(
\Xn^{(m)}_{\mid\mathcal{I}}\right) ^{\dagger}.
\end{equation}

For simplicity of notation, in Sections \ref{app:lemmas_1}-\ref{app:tech} we fix a machine $m$ and thus omit the index $m$ from the proofs.

In our proofs we shall use classical tail bounds
for the Gaussian distribution (Lemma \ref{lem:GaussianTailBounds}), 
a technical lemma regarding the Gaussian distribution, Lemma \ref{lem:two_term_tail}, whose proof appears in Section \ref{app:tech}, and Lemma
\ref{lem:lopes}, which bounds the left tail probability of the maximum of correlated Gaussian random variables \citep{lopes2022sharp}.

\begin{lemma}[Gaussian tail bounds
	\citep{gordon1941values}]\label{lem:GaussianTailBounds}
	For any $t>0$,
	\begin{equation}\label{eq:GaussianTailBoundUpper}
	\frac{t}{\sqrt{2\pi}(t^2+1)}e^{-t^{2}/2} \leq
	\Phi^c(t)\leq\frac{1}{\sqrt{2\pi}t}e^{-t^{2}/2}. 
	\end{equation}  
\end{lemma}

\begin{lemma}\label{lem:two_term_tail}
	For any $a,b\geq 0$, 
	\[
	\Phi^{c}\left(a+b\right) <\sqrt{2}e^{-b^2/2}\Phi^{c}\left(a\right).
	\]
\end{lemma}

\begin{lemma}[\citep{lopes2022sharp}]\label{lem:lopes} 
	Let $\left(Z_{1},\dots,Z_{d}\right)\sim \Nor\left(\bm{0},\bm{\Sigma}\right)$
	where $\bm{\Sigma}_{ii}=1$ for all $i\in[d]$ and $ \bm{\Sigma}_{ij} \leq \eta<1
	$ for some fixed $\eta>0$ for all $i\neq j\in[d]$. 
	Fix $\zeta\in (0,1)$. 
	There is a constant $C>0$ depending only on $(\eta, \zeta)$ such that
	\begin{equation}
	\Pr\left[\max_{i\in [d]} Z_i<\zeta\sqrt{2(1-\eta)\log d}\right]\leq
	Cd^{-\frac{(1-\eta)(1-\zeta)^2}{\eta}}(\log
	d)^{\frac{1-\eta(2-\zeta)-\zeta}{2\eta}} .
	\end{equation}
\end{lemma}
To put Lemma \ref{lem:lopes} in context, recall that the maximum of $d$ independent Gaussians is sharply concentrated at $\sqrt{2\log d}$. 
In general, for correlated Gaussian random variables, their maximum is lower. However, as the lemma shows, it is unlikely to be much lower than $\sqrt{2(1-\eta)\log d}$, where $\eta$ is an upper bound on the correlation. 
We use this result with $\eta=\mumax$ 
and $\zeta=1-\epsilon$,
where $\epsilon$ satisfies Assumption \ref{assum:eps}, 
in order to bound the probability that a non-support index is sent to the center
and prove Lemma \ref{lem:p_j_t1}.

Since here we are considering the case $K=1$, the support of $\tv$ is a single index $\S=\{k\}$. 
In this case, omitting the index of machine $m$, by Eq. \eqref{eq:norm_model} its response vector $\y=\y^{(m)}$
admits the following form
\begin{equation}\label{eq:norm_model1}
\y =  \tn_k\xn_k +\sigma\err.
\end{equation}
Recall that by its definition in Eq. \eqref{eq:tmaxn},  $\tmaxn=\norm{\x_k}\left|
\theta_k\right| = |\tn_k| $. By Eq. \eqref{eq:rho} for $\rho$ and Eq. \eqref{eq:theta_crit} for $\theta_{\crit}$ with $K=1$,
\begin{equation}\label{eq:tmaxn_1}
\tmaxn=\frac{\sigma\sqrt{2\rho\log d}}{1-\mumax}.
\end{equation}
We now prove the lemmas.

\begin{proof}[Proof of Lemma \ref{lem:pmax1}]
	Recall that $p_{\kmax}$, defined in Eq. \eqref{eq:p_j_m}, is the probability
	that the support index $\kmax$ is selected by \OMPStep.
	This occurs if out of all columns of $\Xn^{(m)}$, the $\kmax$-th column has the
	highest correlation with the response vector.
	Hence,	to prove the lemma we need to lower bound the probability of the following event, 
	\begin{equation}\label{eq:pmax_def}
	\left|\left\langle \xn_{\kmax},\y\right\rangle \right|\geq\max_{i\notin \S}\left|\left\langle \xn_{i},\y\right\rangle \right|.
	\end{equation}
	where $\y$ is given by \eqref{eq:norm_model1}. 
	To this end, we decompose the noise $\err$ in Eq. \eqref{eq:norm_model1} as the sum of two
	components, the first
	$\err_{\parallel}=\P_{\kmax}\err=\left\langle \xn_{\kmax},\err\right\rangle \xn_{\kmax}$ is parallel to $\xn_{\kmax}$, namely $\left\langle \xn_{\kmax},\err_{\parallel}\right\rangle =\left\langle \xn_{\kmax},\err\right\rangle $, and the second
	$\err_{\perp}=\err-\err_{\parallel}=\left(\bm{I}-\P_{\kmax}\right)\err$,
	is orthogonal to $\xn_{\kmax}$, i.e., $\left\langle \xn_{\kmax},\err_{\perp}\right\rangle =0$. 
	
	Next, we use this decomposition to bound the two terms in
	\eqref{eq:pmax_def}.
	Combining the expression \eqref{eq:norm_model1} for $\y$, the decomposition of $\err$ and the fact that $\tmaxn=|\tn_k|$, 
	the LHS of \eqref{eq:pmax_def} can be bounded by
	\begin{eqnarray}
	\left|\left\langle \xn_{\kmax},\y\right\rangle \right| & \geq & \sign\left(\tn_{\kmax}\right)\left\langle \xn_{\kmax},\y\right\rangle
	=  \sign\left(\tn_{\kmax}\right)\left(\tn_{k}\left\langle \xn_{\kmax},\xn_{k}\right\rangle +\sigma\left\langle \xn_{\kmax},\err\right\rangle \right)\nonumber\\
	& = & \tmaxn+\sigma\sign\left(\tn_{\kmax}\right)\left\langle \xn_{\kmax},\err_{\parallel}\right\rangle. 
	\label{eq:T3}
	\end{eqnarray}
	Similarly, the RHS of \eqref{eq:pmax_def} can be bounded by
	\begin{eqnarray}
	\max_{i\notin \S}\left|\left\langle \xn_{i},\y\right\rangle \right| 
	& = & \max_{i\notin \S}\left|\tn_{k}\left\langle \xn_{i},\xn_{k}\right\rangle +\sigma\left\langle \xn_{i},\err_{\parallel}+\err_{\perp}\right\rangle \right|\nonumber\\
	& \leq & \left(\tmaxn+\sigma\left|\left\langle \xn_{\kmax},\err\right\rangle \right|\right)\max_{i\notin \S}\left\{ \left|\left\langle \xn_{i},\xn_{\kmax}\right\rangle \right|\right\} +\sigma\max_{i\notin \S}\left|\left\langle \xn_{i},\err_{\perp}\right\rangle \right|\nonumber\\
	& \leq & \left(\tmaxn+\sigma\left|\left\langle \xn_{\kmax},\err_{\parallel}\right\rangle \right|\right)\mumax+\sigma\max_{i\notin \S}\left|\left\langle \xn_{i},\err_{\perp}\right\rangle \right|.		
	\label{eq:T4}
	\end{eqnarray}
	where the second step
	follows from the triangle inequality and the definitions of $\tmaxn$
	and $\err_{\parallel}$, and the last step follows from the definition of
	$\mumax$.
	Combining Eq. \eqref{eq:T3} with Eq. \eqref{eq:T4} implies that a sufficient condition for 
	\eqref{eq:pmax_def} to hold is that 
	\[
	\max_{i\notin \S}\left|\left\langle \xn_{i},\err_{\perp}\right\rangle \right|\leq\sign\left(\tn_{\kmax}\right)\left\langle \xn_{\kmax},\err_{\parallel}\right\rangle -\mumax\left|\left\langle \xn_{\kmax},\err_{\parallel}\right\rangle \right|+\frac{\tmaxn}{\sigma} \left(1-\mumax\right).
	\]
	By Eq. \eqref{eq:tmaxn_1}, the above event may be written as
	\begin{equation}\label{eq:pmax_interm}
	\max_{i\notin \S}\left|\left\langle \xn_{i},\err_{\perp}\right\rangle \right|\leq\sign\left(\tn_{\kmax}\right)\left\langle \xn_{\kmax},\err_{\parallel}\right\rangle -\mumax\left|\left\langle \xn_{\kmax},\err_{\parallel}\right\rangle \right|+\sqrt{2\rho\log d}.
	\end{equation}
	A key property is that 
	$\err_{\parallel}$ and $\err_{\perp}$ are independent random variables. 
	Hence, the left-hand side and right-hand side in the above inequality, which we denote by $A$ and $B$, respectively, are also independent random variables. Now, for any threshold $T\in \R$, with $A,B$ independent
	random variables, 
	\begin{equation}\label{eq:ABT1}
	\Pr\left[A\leq B\right]\geq\Pr\left[A\leq T \cap B\geq T\right]=\Pr\left[A\leq T\right]\cdot\Pr\left[B\geq
	T\right].
	\end{equation}
	Thus, 
	\begin{equation}\label{eq:ABT}
	p_k \geq \Pr[A \leq T] \cdot \Pr]B \geq T]
	\end{equation}
	and it suffices to lower bound these two probabilities.
	
	In what follows we consider $T=\sqrt{2\log d}$. We begin with bounding  $\Pr\left[ A\leq \sqrt{2\log d}\right] $. 
	Fix $i\notin \S$ and consider the quantity $\left\langle \xn_{i},\err_{\perp}\right\rangle$. 
	We may write $\xn_i = \P_k \xn_i + (\I-\P_k)\xn_i$
	Since 
	$\err_{\perp}=\left(\bm{I}-\P_{\kmax}\right)\err$, then  
	$\langle \P_k \xn_i, \err_{\perp}\rangle = 0$, and 
	$\left\langle \xn_{i},\err_{\perp}\right\rangle =
	\left\langle \left(\bm{I}-\P_{\kmax}\right)\xn_{i},\err_{\perp}\right\rangle .$
	Normalizing the inner product by the norm of $(\I-\P_k) \xn_i$ yields a standard normal random variable
	$Z_{i}=\frac{\left\langle \xn_{i},\err_{\perp}\right\rangle }{\left\Vert \left(\bm{I}-\P_{\kmax}\right)\xn_{i}\right\Vert _{2}}\sim\Nor(0,1).$
	By the definition of $\mumax$, 
	\[
	\left\Vert \left(\bm{I}-\P_{\kmax}\right)\xn_{i}\right\Vert ^{2}=\xn_{i}^{T}\left(\bm{I}-\P_{\kmax}\right)\xn_{i}=1-\left\langle \xn_{i},\xn_{\kmax}\right\rangle ^{2}\geq\gamma_1^2,
	\]
	where $\gamma_1=\sqrt{1-\mumax^{2}}$.
	Hence, 
	\[
	\Pr\left[ A\leq T\right] \geq\Pr\left[ \max_{i\notin \S}\left|Z_{i}\right|\leq\frac{T}{\gamma_1}\right] .
	\]
	Since $\left\lbrace Z_i\right\rbrace_{i\notin \S} $ are jointly Gaussian, by
	\cite[Thm. 1]{vsidak1967rectangular}, regardless of their covariance structure, 
	\[
	\Pr\left[\max_{i\notin \S}\left|Z_{i}\right|\leq\frac{T}{\gamma_1}\right]\geq\prod_{i\notin \S}\Pr\left[\left|Z_{i}\right|\leq\frac{T}{\gamma_1}\right].
	\]
	Applying the Gaussian tail bound \eqref{eq:GaussianTailBoundUpper} with
	$T=\sqrt{2\log d}$,
	\[		
	\Pr\left[\left|Z_{i}\right|\leq\frac{\sqrt{2\log d}}{\gamma_1}\right]\geq
	1-\frac{\gamma_1}{\sqrt{\pi\log d}}d^{-1/\gamma_1^2}.
	\]
	Combining the above three inequalities with Bernoulli's inequality 
	$(1-a)^d\geq 1-da$ which holds for any $a\in [0,1]$, gives
	\begin{equation}\label{eq:PrA}
	\Pr\left[A\leq \sqrt{2\log d}\right] 
	\geq \left(1-\frac{\gamma_1}{\sqrt{\pi\log d}}d^{-1/\gamma_1^2}\right)^{d-1}
	\geq 1-\frac{\gamma_1}{\sqrt{\pi\log d}}d^{1-1/{\gamma_1^{2}}}\geq \frac{1}{2},
	\end{equation}
	where the last inequality holds for sufficiently large $d$ and follows from noting that $0<\gamma_1\leq 1$.
	
	We now bound $\Pr\left[ B\geq T\right] $, where $B$ is the RHS of \eqref{eq:pmax_interm}. 
	Since $\xn_k$ has unit norm, by the definition of
	$\err_{\parallel}$, then $Z =\left\langle \xn_{\kmax},\err_{\parallel}\right\rangle =\left\langle \xn_{\kmax},\err\right\rangle \sim\Nor\left( 0,1\right) $.
	By the law of total probability, 
	\begin{eqnarray*}
		\Pr\left[B\geq T\right] & = & \Pr\left[\sign\left(\tn_{\kmax}\right)\left\langle \xn_{\kmax},\err_{\parallel}\right\rangle -\mumax\left|\left\langle \xn_{\kmax},\err_{\parallel}\right\rangle \right|\geq T-\sqrt{2\rho\log d}\right]\\
		& \geq & \Pr\left[\gamma_2\left|Z\right|\geq T-\sqrt{2\rho\log d}\mid\sign\left(Z\right)=\sign\left(\tn_{\kmax}\right)\right]\cdot\Pr\left[\sign\left(Z\right)=\sign\left(\tn_{\kmax}\right)\right],
	\end{eqnarray*}
	where $\gamma_2=1-\mumax$.
	Since $Z$ is symmetric around zero,
	$\Pr\left[\sign\left(Z\right)=\sign\left(\tn_{\kmax}\right)\right]=\frac{1}{2}$ 
	and its magnitude is independent on its sign. Thus, for $T=\sqrt{2\log d}$, 
	\begin{equation}\label{eq:PrB}
	\Pr\left[B\geq \sqrt{2\log d}\right]
	\geq
	\frac{1}{2}\Pr\left[\gamma_2\left|Z\right|\geq
	\sqrt{2\log d}-\sqrt{2\rho\log d}\right]
	\geq
	\Phi^{c}\left(\frac{1-\sqrt{\rho}}{\gamma_2}\sqrt{2\log d}\right).
	\end{equation}
	Inserting \eqref{eq:PrA} and \eqref{eq:PrB} with $\gamma_2=1-\mumax$ into \eqref{eq:ABT} and recalling the definition of $\F$ in \eqref{eq:F_1} completes the proof of Lemma \ref{lem:pmax1}. 
\end{proof}

\begin{proof}[Proof of Lemma \ref{lem:p_j_t1}]
	Fix a non-support index $j\notin \S$.
	Recall that $p_{j}$, defined in Eq. \eqref{eq:p_j_m}, is the probability
	that index $j$ is selected by \OMPStep.
	This occurs if $j$ has the
	highest correlation with the response vector, i.e., 
	\begin{equation}\label{eq:j_geq_noise_t1}
	p_j = \Pr\left[ \left|\left\langle \xn_{j},\y\right\rangle \right|>\max_{i\neq j}\left|\left\langle \xn_{i},\y\right\rangle \right|\right] .
	\end{equation}
	In particular, for the $j$-th index to be chosen, the correlation of the $j$-th column with the response
	vector must exceed both that of the support column $\kmax$, as well as that of any other
	non-support column $i\notin \left\lbrace \kmax,j \right\rbrace$.
	Indeed, in what follows we separately upper bound
	\begin{equation}\label{eq:j_gtr_i1}
	\Pr\left[ \left|\left\langle \xn_{j},\y\right\rangle \right|>\max_{i\notin\left\lbrace \kmax,j\right\rbrace }\left|\left\langle \xn_{i},\y\right\rangle \right|\right] 
	\end{equation}
	and
	\begin{equation}\label{eq:j_gtr_kmax1}
	\Pr\left[ \left|\left\langle \xn_{j},\y\right\rangle \right|>\left|\left\langle \xn_{\kmax},\y\right\rangle \right|\right] ,
	\end{equation}
	and then use the following inequality to upper bound \eqref{eq:j_geq_noise_t1} by their
	minimum.
	Specifically, denote $A=\left|\left\langle \xn_{j},\y\right\rangle \right|$, $B=\max_{i\notin\left\lbrace \kmax,j\right\rbrace }\left|\left\langle \xn_{i},\y\right\rangle \right| $ and $C=\left|\left\langle \xn_{\kmax},\y\right\rangle \right|$, then
	\begin{equation}\label{eq:ABC1}
	\Pr\left[A>\max\left\lbrace B,C\right\rbrace  \right]\leq \min \left\lbrace
	\Pr\left[A>B \right], \Pr\left[A>C \right]   \right\rbrace  .
	\end{equation}	
	For later use in both bounds, by the triangle inequality, the random variable
	$A$ can be upper bounded as follows
	\begin{equation}\label{eq:j_bound1}
	\left|\left\langle \xn_{j},\y\right\rangle \right|=\left|\tn_{\kmax}\left\langle \xn_{j},\xn_{\kmax}\right\rangle +\sigma\left\langle \xn_{j},\err\right\rangle \right|\leq\tmaxn\mumax+\sigma\left|\left\langle \xn_{j},\err\right\rangle \right|.
	\end{equation}
	We first bound \eqref{eq:j_gtr_i1}. For each non-support index
	$i\notin \S$ such that $i\neq j$,
	\begin{equation*}
	\left|\left\langle \xn_{i},\y\right\rangle \right|\geq\left\langle \xn_{i},\y\right\rangle =\tn_{\kmax}\left\langle \xn_{i},\xn_{\kmax}\right\rangle +\sigma\left\langle \xn_{i},\err\right\rangle \geq-\tmaxn\mumax+\sigma\left\langle \xn_{i},\err\right\rangle .
	\end{equation*}
	Combining this with Eq. \eqref{eq:j_bound1}, rearranging terms, and
	recalling the relation between $\tmax$ and $\rho$ in \eqref{eq:tmaxn_1} yields
	\begin{equation}\label{eq:T6_1}
	\Pr\left[ \left|\left\langle \xn_{j},\y\right\rangle \right|>\max_{i\notin\left\lbrace \kmax,j\right\rbrace }\left|\left\langle \xn_{i},\y\right\rangle \right|\right]\leq 
	\Pr\left[\left|\left\langle \xn_{j},\err\right\rangle \right|+2\mumax\frac{\sqrt{2\rho\log d}}{1-\mumax}>\max_{i\notin \left\{ k,j\right\} }\left\langle \xn_{i},\err\right\rangle \right].
	\end{equation}

	Next, we use the following inequality which holds for any pair of random variables $D,E$ and constant $T\in\R$,
	\begin{equation}\label{eq:DET}
	\Pr\left[D> E\right]\leq\Pr\left[D\geq T\right]+\Pr\left[E<
	T\right].
	\end{equation}
	Applying this inequality 
	with $T=(1-\epsilon)\sqrt{2\left(1-\mumax\right)\log d}$ and
	$\epsilon\in\left(0,1\right)$ as in Eq. \eqref{eq:eps_cond}, we can upper bound 
	\eqref{eq:T6_1} by 
	\[
	\Pr\left[\left|\left\langle \xn_{j},\err\right\rangle \right|\geq a\sqrt{2\log d}\right]
	+\Pr\left[\max_{i\notin \left\{k, j\right\} }\left\langle \xn_{i},\err\right\rangle <(1-\epsilon)\sqrt{2\left(1-\mumax\right)\log d}\right],
	\]
	where $$a=(1-\epsilon)\sqrt{1-\mumax}-\frac{2\mumax\sqrt{\rho}}{1-\mumax}.$$
	Since $\xn_{j}$ has unit norm, $\left\langle \xn_{j},\err\right\rangle  \sim\Nor\left( 0,1\right) $. Hence, the first term is bounded by 
	\begin{equation}
	\label{eq:first_term1}
	2\Phi^{c}\left(a\sqrt{2\log d}\right).
	\end{equation}
	We now bound the second term. It involves the maximum of $d-2$ correlated
	Gaussians, whose covariance matrix $\Sigma$ has $\Sigma_{ii}=1$ for all $i$, 
	and $\Sigma_{ij} = \Cov\left(\innerp{\xn_i}{\err} , \innerp{\xn_l}{\err}\right)=\innerp{\xn_i}{\xn_l} \leq \mumax $. 
	Hence, we can apply Lemma \ref{lem:lopes} with $\eta=\mumax$ and
	$\zeta=1-\epsilon$, which gives the following bound
	\begin{equation}\label{eq:post_lopes1}
	C\left(d-2\right)^{-\frac{1-\mumax}{\mumax}\epsilon^{2}}(\log
	(d-2))^{\frac{\epsilon-\mumax(1+\epsilon)}{2\mumax}}.
	\end{equation}
	We now show that \eqref{eq:first_term1} is larger than \eqref{eq:post_lopes1},
	and thus 
	\begin{equation}\label{eq:p_non_supp1}
	\Pr\left[ \left|\left\langle \xn_{j},\y\right\rangle \right|>\max_{i\notin\left\lbrace \kmax,j\right\rbrace }\left|\left\langle \xn_{i},\y\right\rangle \right|\right]  \leq 4\Phi^{c}\left(a\sqrt{2\log d}\right).
	\end{equation}
	First note that if $\rho$ is sufficiently large such that $a\leq 0$, then
	\eqref{eq:first_term1} is larger than $1$, and thus larger than
	\eqref{eq:post_lopes1}.
	Otherwise, $a>0$ and using the lower bound for the Gaussian tail of
	\eqref{eq:GaussianTailBoundUpper}, we may lower bound \eqref{eq:first_term1} by 
	$d^{-a^2 -o(1) } $, where $o(1)$ hides factors that are asymptotically smaller
	than 1. The term \eqref{eq:post_lopes1} can be upper bounded by $d^{-b^{2} +o(1)
	} $, where $b=\sqrt{\frac{1-\mumax}{\mumax}}\epsilon$. 
	Next, let us show that for a fixed $\epsilon>0$, $ b- a  $ is positive and bounded away from $0$. 
	This, in turn, implies that for sufficiently large $d=d\left( \epsilon\right) $, \eqref{eq:first_term1} is larger than
	\eqref{eq:post_lopes1}.
	Indeed, under condition \eqref{eq:eps_cond},
	$\epsilon=\frac{\sqrt{\mumax}}{1+\sqrt{\mumax}}+\epsilon_0$ for some
	$\epsilon_0>0$.
	Thus, $b-a=\epsilon_{0}\sqrt{1-\mumax}\left(1+\frac{1}{\sqrt{\mumax}}\right)+\frac{2\mumax\sqrt{\rhot}}{1-\mumax}$, which is a sum of positive terms and hence bounded away from 0 as desired.
	Therefore, condition \eqref{eq:eps_cond} implies that 
	\eqref{eq:j_gtr_i1} can be bounded by \eqref{eq:p_non_supp1}.

	We now bound \eqref{eq:j_gtr_kmax1}.
	For the support index $\kmax$, by \eqref{eq:norm_model1},
	\begin{equation*}
	\left|\left\langle \xn_{\kmax},\y\right\rangle \right|  \geq  \sign\left(\tn_{\kmax}\right)\left\langle \xn_{\kmax},\y\right\rangle =\sign\left(\tn_{\kmax}\right)\left(\tn_{\kmax}\left\langle \xn_{\kmax},\xn_{\kmax}\right\rangle +\sigma\left\langle \xn_{\kmax},\err\right\rangle \right)
	= \tmaxn+\sigma\sign\left(\tn_{\kmax}\right)\left\langle \xn_{\kmax},\err\right\rangle .
	\end{equation*}
	Combining this with \eqref{eq:j_bound1} and plugging $\tmaxn$ in Eq. \eqref{eq:tmaxn_1}, the probability (\ref{eq:j_gtr_kmax1}) is upper bounded by 
	\begin{equation}\label{eq:j_gtr_kmax1_2}
	\Pr\left[\left|\left\langle \xn_{j},\err\right\rangle \right|-\sign\left(\tn_{\kmax}\right)\left\langle \xn_{\kmax},\err\right\rangle >\sqrt{2\rho\log d}\right].
	\end{equation}
	We now upper bound this probability.
	Let $H=\left\langle \xn_{j},\err\right\rangle $,
	$G=\sign\left(\tn_{\kmax}\right)\left\langle \xn_{\kmax},\err\right\rangle $ and $c=\sqrt{2\rho\log d}$. 
	
	For any pair of random variables $G,H$ and constant $c$,
	\begin{equation}\label{eq:int1_1}
	\Pr\left[\left|H\right|-G>c\right]\leq \Pr\left[H-G>c\right] +
	\Pr\left[-H-G>c\right].
	\end{equation}
	By their definition, $H,G$ are jointly Gaussian with mean zero and covariance matrix 
	$$\left(\begin{array}{cc}
	\sigma_{H}^{2} & \sigma_{HG}\\
	\sigma_{HG} & \sigma_{G}^{2}
	\end{array}\right).$$ 
	Hence, $H-G\sim \Nor (0,\sigma_{H}^{2}+\sigma_{G}^{2}-2\sigma_{HG})$ and
	$-H-G\sim \Nor (0,\sigma_{H}^{2}+\sigma_{G}^{2}+2\sigma_{HG})$. 	
	Similarly to the above discussion, the diagonal entries $\sigma_{H}^{2}=\sigma_{G}^{2}=1 $ and the off-diagonal entry $\sigma_{HG}=\sign\left(\tn_{\kmax}\right)\left\langle \xn_{\kmax},\xn_{j}\right\rangle $. 
	Since $\left| \sigma_{HG}\right| \leq \mumax$, then by \eqref{eq:int1_1},
	\begin{equation*}\label{eq:HGc1}
	\Pr\left[\left|H\right|-G>c\right]\leq\Phi^{c}\left(\frac{c}{\sqrt{2-2\sigma_{HG}}}\right)+\Phi^{c}\left(\frac{c}{\sqrt{2+2\sigma_{HG}}}\right)\leq2\Phi^{c}\left(\frac{c}{\sqrt{2+2\mumax}}\right).
	\end{equation*}
	Inserting $c=\sqrt{2\rho\log d}$ yields 
	\begin{equation}\label{eq:sig_bound1}
	\Pr\left[ \left|\left\langle \xn_{j},\y\right\rangle \right|>\left|\left\langle \xn_{\kmax},\y\right\rangle \right|\right] \leq 2\Phi^c\left( \sqrt{\frac{2\rho\log d}{2+2\mumax}}\right).
	\end{equation}
	By Eq. \eqref{eq:ABC1}, the probability \eqref{eq:j_geq_noise_t1} is at
	most the minimum between \eqref{eq:p_non_supp1} and \eqref{eq:sig_bound1}. By
	the monotonicity of the Gaussian CDF, it is upper bounded by 
	\begin{equation}\label{eq:T7}
	4\Phi^{c}\left(\max\left\lbrace \left((1-\epsilon)\sqrt{1-\mumax}-\frac{2\mumax\sqrt{\rho}}{1-\mumax}\right),\sqrt{\frac{\rho}{2+2\mumax}}\right\rbrace \sqrt{2\log d}\right).
	\end{equation}
	
	Finally, to prove \eqref{eq:p_j_t_m_bound1} of the lemma, we note that with
	$Q_1$ and $Q_2$ defined in Eqs. \eqref{eq:Q_1} and \eqref{eq:Q_2}
	respectively, 
	by splitting to cases and applying some algebraic manipulations\footnote{First, consider the case $\mumax\geq 1/2$. By the max-MIP condition \eqref{eq:max_mip_cond}, $\mumax<1$, and hence the term $\frac{1-\mumax+\sqrt{2+2\mumax}}{(1-\mumax)\sqrt{2+2\mumax}}$ is positive, and thus can multiply both sides of the inequality $\sqrt{r}>Q_2$ without altering its direction. Rearranging yields that the LHS of \eqref{eq:r_cond2} is smaller than $\sqrt{\frac{r}{2+2\mumax}}$ and thus smaller than the RHS of \eqref{eq:r_cond2}. Now consider the case  $\mumax< 1/2$. By \eqref{eq:r_cond}, $\sqrt{r}>Q_1 $ or $\sqrt{r}>Q_2 $. By the same reasoning, the latter implies that the LHS of \eqref{eq:r_cond2} is smaller than $\sqrt{\frac{r}{2+2\mumax}}$. Similarly, the term $\frac{1-2\mumax}{1-\mumax}$ is positive in this case, and thus multiplying the inequality $\sqrt{r}>Q_1 $ by it and rearranging the terms implies that the LHS of \eqref{eq:r_cond2} is smaller than $(1-\epsilon)\sqrt{1-\mumax}-\frac{2\mumax\sqrt{r}}{1-\mumax}$. Finally, the logical or relation between these conditions implies that the LHS of \eqref{eq:r_cond2} is smaller than the maximum between the aforementioned terms.
	},
	condition \eqref{eq:r_cond} 
	implies that
	\begin{equation}\label{eq:r_cond2}
	\frac{1-\sqrt{r}}{1-\mumax}+\sqrt{Q_{0}}<\max\left\lbrace
	\left( (1-\epsilon)\sqrt{1-\mumax}-\frac{2\mumax\sqrt{r}}{1-\mumax}\right) ,\sqrt{\frac{r}{2+2\mumax}}\right\rbrace.
	\end{equation}
	The definitions of $r$ and $\rho$ in Eqs. \eqref{eq:r_def} and \eqref{eq:rho} imply that $\rhot\geq r$. Thus, $\rho$ satisfies condition \eqref{eq:r_cond} and hence condition
	\eqref{eq:r_cond2}.
	The RHS of  \eqref{eq:r_cond2} is the same as the
	maximum in  \eqref{eq:T7} above.
	Thus,  \eqref{eq:T7} is upper bounded by
	\begin{equation}
	4\Phi^c\left( \left( \frac{1-\sqrt{\rhot}}{1-\mumax} +\sqrt{Q_0}\right)
	\sqrt{2\log d}\right).
	\end{equation}
	Since $\rho\leq 1$, we can apply Lemma \ref{lem:two_term_tail}. Hence, by the definition of
	$Q_0$ in Eq. \eqref{eq:Q_0}, and by the definition of $\F$ in Eq. \eqref{eq:F_1},  
	\begin{eqnarray*}			
		p_{j} & \leq & 4\Phi^{c}\left(\left(\frac{1-\sqrt{\rhot}}{1-\mumax}+\sqrt{Q_{0}}\right)\sqrt{2\log d}\right)\leq4\sqrt{2}d^{-Q_{0}}\Phi^{c}\left(\frac{1-\sqrt{\rhot}}{1-\mumax}\sqrt{2\log d}\right)\\
		&  & =4\sqrt{2}\frac{1}{88\sqrt{2}}\Phi^{c}\left(\frac{1-\sqrt{\rhot}}{1-\mumax}\sqrt{2\log d}\right)=\frac{\F\left(d,1,\mumax,\rhot\right)}{11},
	\end{eqnarray*}
	which completes the proof of Lemma \ref{lem:p_j_t1}.
\end{proof}

\subsection{Proof of Lemmas \ref{lem:pmax} and \ref{lem:p_j_t}}\label{app:lemmas_K}
We first make a few definitions and present a useful technical lemma.
We begin by rewriting the residual $\res^{(m)}$ using the notations introduced in Section \ref{app:lemmas_1}.
Recall that given an input support set $\hats$, each machine $m$ estimates its vector $ \hat{\tv}^{(m)}$ by solving the least squares problem \eqref{eq:hat_tv}. 
Thus, $\supp\left( \hat{\tv}^{(m)}\right) =\hats$ and
\[
\hat{\tv}^{(m)}_{\mid \hats} = \left( \X^{(m)}_{\mid\hats}\right) ^{\dagger} \y^{(m)}.
\]
Denote by $\erre^{(m)}$ the projection of the noise $\err^{(m)}$ to the subspace orthogonal to the span of the columns of $\X^{(m)}_{\mid\hats}$, i.e., 
$\erre^{(m)}=\left(\bm{I}-\P_{\hats}^{(m)}\right)\err^{(m)}$. 
Given that $\hats\subset \S$, the residual $\res^{(m)}$ defined in Eq. \eqref{eq:res_m} can be written in the following form
\begin{eqnarray}
\res^{(m)} & = & \y^{(m)}-\X^{(m)}\hat{\tv}^{(m)}=\y^{(m)}-\X_{\mid\hats}^{(m)}\hat{\tv}_{\mid\hats}^{(m)}=\left(\I-\X_{\mid\hats}^{(m)}\left(\X_{\mid\hats}^{(m)}\right)^{\dagger}\right)\y^{(m)}\nonumber\\
& = & \left(\bm{I}-\P_{\hats}^{(m)}\right)\y^{(m)}=\left(\bm{I}-\P_{\hats}^{(m)}\right)\left(\Xn^{(m)}\tvn^{(m)}+\sigma\err^{(m)}\right)\nonumber\\
& = & \left(\bm{I}-\P_{\hats}^{(m)}\right)\sum_{l\in \S\setminus{\hats}}\tn_{l}^{(m)}\xn_{l}^{(m)}+\sigma\erre^{(m)},
\label{eq:res_m_K}
\end{eqnarray}
where $\Xn^{(m)}$ and $\tvn^{(m)}$ are the scaled versions of $\X^{(m)}$ and $\tv$, as discussed after Eq. \eqref{eq:norm_model}, and the last equality follows from the definition of $\P_{\hats}^{(m)}$ as a projection operator, so that $\left(\bm{I}-\P_{\hats}^{(m)}\right)\xn_{k}^{(m)}=\bm{0}$ for any $k\in\hats$.

Denote by $\Kd$ the size of the detected support set, i.e., $\Kd=\left| \hats\right| $, and by $\Ku$ the size of the undetected support set, i.e., $\Ku=|\S\setminus{\hats}|$. Since $\hats\subset \S$, then $\Kd+\Ku=K$.
Finally, we introduce the following quantity 
\begin{equation}\label{eq:mud}
\mud=\mud\left(\Kd,\mumax \right)  = \frac{\Kd\mumax^2}{1-(\Kd-1)\mumax}.
\end{equation}
The following Lemma \ref{lem:projs} bounds the effect of the projection $\bm{I}-\P_{\hats}^{(m)}$ on the inner products and norms of columns of $\Xn^{(m)}$.
Its proof appear in Appendix \ref{app:tech}.

\begin{lemma}\label{lem:projs}
	Assume that the max-MIP condition \eqref{eq:max_mip_cond} holds and that $\hats\subset \S$.
	Then, the following inequalities hold for any $0\leq \Kd\leq K-1$ and $1\leq \Ku \leq K$ such that $\Kd+\Ku=K$:
	\begin{enumerate}
		\item The quantity $\mud$ of Eq. \eqref{eq:mud} satisfies
		\begin{equation}\label{eq:mud_mumax}
		\mud\leq \mumax,
		\end{equation}
		and
		\begin{equation}\label{eq:tmu_mip}
		\Ku\left(\mumax+\mud\right)< K\mumax.
		\end{equation}
		\item For each index $i\notin \hats$,
		\begin{equation}\label{eq:proj_I_PA}
		1-\mud\leq \left\Vert \left(\bm{I}-\P_{\hats}^{(m)}\right)\xn_{i}^{(m)}\right\Vert _{2}^{2}\leq 1.
		\end{equation}
		\item For each pair of distinct indices $i\neq k$ such that $i,k\notin \hats$,
		\begin{equation}\label{eq:inprod_I_PA}
		\left|  \left\langle \xn_{k}^{(m)}, \left(\bm{I}-\P_{\hats}^{(m)}\right)\xn_{i}^{(m)}\right\rangle\right|  \leq \mumax+\mud,
		\end{equation}
		and
		\begin{equation}\label{eq:proj_I_Pk_PA}
		\left\Vert \left(\bm{I}-\P_{\hats}^{(m)}\right)\left(\bm{I}-\P_{k}^{(m)}\right)\xn_{i}^{(m)}\right\Vert _{2}^{2}\geq 1-\mumax^{2}-\mud\left(1+\mumax\right)^{2}.
		\end{equation}
		Furthermore, $1-\mumax^{2}-\mud\left(1+\mumax\right)^{2}>0.$
	\end{enumerate}
\end{lemma}

For future use, notice that by its definition in Eq. \eqref{eq:mud}, $\mud$ is an increasing function of $\Kd$. Since $\Kd\leq K-1$, then the quantity $\delta$ of Eq. \eqref{eq:delta} satisfies
\begin{equation}\label{eq:mud_bound}
\delta\left(K,\mumax \right) = \mud\left(K-1,\mumax \right) \geq \mud\left(\Kd,\mumax \right)  .
\end{equation}
In addition, by Eq. \eqref{eq:mud_mumax}, under max-MIP condition \eqref{eq:max_mip_cond}, $\delta\leq\mumax<1$, and hence the quantities in Section \ref{sec:theory} are well defined.

For simplicity of notation, from now on we omit the dependence on the machine index $m$.
Given the current estimated support set $\hats$, recall the definition of $\tmaxn$ in Eq. \eqref{eq:tmaxn} and let $\kmax\in \S\setminus \hats$ be an index for which 
\begin{equation}\label{eq:kmax}
\norm{\xn_k}\cdot\abs{\theta_k}=\tmaxn
\end{equation}
(chosen arbitrarily in case of ties).
By Eq. \eqref{eq:rho} for $\rho$ and Eq. \eqref{eq:theta_crit} for $\theta_{\crit}$,
\begin{equation}\label{eq:tmaxn_K}
\tmaxn=\frac{\sigma\sqrt{2\rho\log d}}{1-(2K-1)\mumax}.
\end{equation}
We now prove Lemmas \ref{lem:pmax} and \ref{lem:p_j_t}.	
\begin{proof}[Proof of Lemma \ref{lem:pmax}]
	Recall that $q$, defined in Eq. \eqref{eq:qm_def}, is the probability that some support index is selected by \OMPStep.
	A sufficient condition for this to occur is that the index $\kmax$ defined in Eq. \eqref{eq:kmax} has a
	higher correlation with the current residual than any non-support index
	$j\notin \S$. Thus, $q$ is lower bounded by the probability of the following event
	\begin{equation}\label{eq:pmax_defK}
	\left|\left\langle \xn_{\kmax},\res\right\rangle \right|\geq\max_{i\notin \S}\left|\left\langle \xn_{i},\res\right\rangle \right|.
	\end{equation}
	Thus, to prove the lemma it suffices to lower bound the probability of event \eqref{eq:pmax_defK}.
	Similarly to the proof of Lemma \ref{lem:pmax1}, we decompose the noise $\erre$ in Eq. \eqref{eq:res_m_K} as the sum of two
	components, the first
	$\erre_{\parallel}=\P_{\kmax}\erre=\left\langle \xn_{\kmax},\erre\right\rangle \xn_{\kmax}$ is parallel to $\xn_{\kmax}$, namely $\left\langle \xn_{\kmax},\erre_{\parallel}\right\rangle =\left\langle \xn_{\kmax},\erre\right\rangle $, and the second
	$\erre_{\perp}=\erre-\erre_{\parallel}=\left(\bm{I}-\P_{\kmax}\right)\erre$,
	is orthogonal to $\xn_{\kmax}$, i.e., $\left\langle \xn_{\kmax},\erre_{\perp}\right\rangle =0$.

	Next, we use this decomposition to bound each of the terms in \eqref{eq:pmax_defK}.
	By Eq. \eqref{eq:res_m_K}, for any index $i$,
	\[\left\langle \xn_{i},\res\right\rangle =\sum_{l\in \S\setminus\hats}\tn_{l}\left\langle \xn_{i},\left(\bm{I}-\P_{\hats}\right)\xn_{l}\right\rangle +\sigma\left\langle \xn_{i},\erre\right\rangle. \]
	For the index $\kmax$, $\left\Vert \left(\bm{I}-\P_{\hats}\right)\xn_{k}\right\Vert _{2}^{2}=\left\langle \xn_{\kmax},\left(\bm{I}-\P_{\hats}\right)\xn_{\kmax}\right\rangle \geq1-\mud$ by Eq. \eqref{eq:proj_I_PA}. For any other undetected support index $l\in \S\setminus\left\lbrace \hats \cup \kmax \right\rbrace $, $\left|\left\langle \xn_{k},\left(\bm{I}-\P_{\hats}\right)\xn_{l}\right\rangle \right|\leq\mumax+\mud$ by \eqref{eq:inprod_I_PA} and $\abs{\tn_l}\leq \tmaxn$ by its definition in Eq. \eqref{eq:tmaxn}. Combining these bounds with the definition of $\erre_{\parallel}$ implies that the LHS of \eqref{eq:pmax_defK} can be bounded by	
	\begin{eqnarray}
	\left|\left\langle \xn_{\kmax},\res\right\rangle \right| & \geq & \sign\left(\tn_{\kmax}\right)\left\langle \xn_{\kmax},\res\right\rangle \nonumber\\
	& \geq & \tmaxn\left(\left\langle \xn_{\kmax},\left(\bm{I}-\P_{\hats}\right)\xn_{\kmax}\right\rangle -\sum_{l\in \S\setminus\left(\hats\cup\left\{ \kmax\right\} \right)}\left|\left\langle \xn_{\kmax},\left(\bm{I}-\P_{\hats}\right)\xn_{l}\right\rangle \right|\right)+\sign\left(\tn_{\kmax}\right)\sigma\left\langle \xn_{\kmax},\erre\right\rangle \nonumber\\
	& \geq & \tmaxn\left(1-\mud-\left(\Ku-1\right)\left(\mumax+\mud\right)\right)+\sigma\sign\left(\tn_{\kmax}\right)\left\langle \xn_{\kmax},\erre_{\parallel}\right\rangle . \label{eq:T3K}
	\end{eqnarray}
	The RHS of \eqref{eq:pmax_defK} can be bounded by
	\begin{eqnarray}
	\max_{i\notin \S}\left|\left\langle \xn_{i},\res\right\rangle \right| & = & \max_{i\notin \S}\left|\sum_{l\in \S\setminus\hats}\tn_{l}\left\langle \xn_{i},\left(\bm{I}-\P_{\hats}\right)\xn_{l}\right\rangle +\sigma\left\langle \xn_{i},\erre_{\perp}+\erre_{\parallel}\right\rangle \right|\nonumber\\
	& \leq & \tmaxn\max_{i\notin \S}\sum_{l\in \S\setminus\hats}\left|\left\langle \xn_{i},\left(\bm{I}-\P_{\hats}\right)\xn_{l}\right\rangle \right|+\sigma\max_{i\notin \S}\left|\left\langle \xn_{i},\erre_{\perp}\right\rangle \right|+\sigma\left|\left\langle \xn_{\kmax},\erre\right\rangle \right|\max_{i\notin \S}\left|\left\langle \xn_{i},\xn_{\kmax}\right\rangle \right|\nonumber\\
	& \leq & \Ku\tmaxn\left(\mumax+\mud\right)+\sigma\max_{i\notin \S}\left|\left\langle \xn_{i},\erre_{\perp}\right\rangle \right|+\sigma\mumax\left|\left\langle \xn_{\kmax},\erre_{\parallel}\right\rangle \right|.\label{eq:T4K}
	\end{eqnarray}
	where the first step follows from Eq. \eqref{eq:res_m_K} and the definitions of $\erre_{\perp}$ and $\erre_{\parallel}$, the second step follows from the triangle inequality and the definitions of $\erre_{\parallel}$ and $\tmaxn$, and the last inequality follows from Eq. \eqref{eq:inprod_I_PA} and the definitions of $\mumax$ in Eq. \eqref{eq:mumax}.
	Combining Eq. \eqref{eq:T3K} with Eq. \eqref{eq:T4K} implies that  a sufficient condition for \eqref{eq:pmax_defK} to occur is
	\[
	\max_{i\notin \S}\left|\left\langle \xn_{i},\erre_{\perp}\right\rangle \right|\leq\sign\left(\tn_{\kmax}\right)\left\langle \xn_{\kmax},\erre_{\parallel}\right\rangle -\mumax\left|\left\langle \xn_{\kmax},\erre_{\parallel}\right\rangle \right|+\frac{\tmaxn\left(1-2\Ku\left(\mumax+\mud\right)+\mumax\right)}{\sigma}.
	\]
	By Eq. \eqref{eq:tmaxn_K} and by the inequality \eqref{eq:tmu_mip}, a sufficient condition for the previous event to occur is
	\begin{equation}\label{eq:pmax_intermK}
	\max_{i\notin \S}\left|\left\langle \xn_{i},\erre_{\perp}\right\rangle \right|\leq\sign\left(\tn_{\kmax}\right)\left\langle \xn_{\kmax},\erre_{\parallel}\right\rangle -\mumax\left|\left\langle \xn_{\kmax},\erre_{\parallel}\right\rangle \right|+\sqrt{2\rhot\log d}.
	\end{equation}
	As in the proof of Lemma \ref{lem:pmax1}, denote the LHS of \eqref{eq:pmax_intermK} by $A$, its RHS by $B$ and let $T=\sqrt{2\log d}$. By Eq. \eqref{eq:ABT1}, it suffices to bound the probabilities of $A\leq T$ and $B\geq T$.

	We begin with bounding $\Pr\left[ A\leq T\right] $. 
	Fix $i\notin \S$. By definition $\erre_{\perp}=\left(\bm{I}-\P_{\kmax}\right)\erre=\left(\bm{I}-\P_{\kmax}\right)\left(\bm{I}-\P_{\hats}\right)\err$. 
	By the symmetry of projections,
	$
	\left\langle \xn_{i},\erre_{\perp}\right\rangle =\left\langle \xn_{i},\left(\bm{I}-\P_{\kmax}\right)\left(\bm{I}-\P_{\hats}\right)\err\right\rangle =\left\langle \left(\bm{I}-\P_{\hats}\right)\left(\bm{I}-\P_{\kmax}\right)\xn_{i},\err\right\rangle .
	$
	Normalizing the inner product results in a standard normal random variable $Z_{i}=\frac{\left\langle \xn_{i},\erre_{\perp}\right\rangle }{\left\Vert \left(\bm{I}-\P_{\hats}\right)\left(\bm{I}-\P_{\kmax}\right)\xn_{i}\right\Vert _{2}}\sim\Nor(0,1).$
	By Eq. \eqref{eq:proj_I_Pk_PA}, 
	$\left\Vert \left(\bm{I}-\P_{\hats}\right)\left(\bm{I}-\P_{\kmax}\right)\xn_{i}\right\Vert _{2} \geq \gamma_1,$
	where $\gamma_1 = \sqrt{1-\mumax^{2}-\mud\left(1+\mumax\right)^{2}}$.
	As in the proof of Lemma \ref{lem:pmax1}, it follows that
	\begin{equation}\label{eq:PrAK}
	\Pr\left[A\leq T\right]\geq1-\frac{\gamma_1}{\sqrt{\pi\log d}}d^{-\frac{1}{\gamma_1^{2}}+1}\geq \frac{1}{2},
	\end{equation}
	where the last inequality holds for sufficiently large $d$  and follows from noting that $\gamma_1\leq 1$ by the max-MIP condition \eqref{eq:max_mip_cond}.

	We now bound  $\Pr\left[ B\geq T\right] $, where $B$ is the RHS of Eq. \eqref{eq:pmax_intermK}. By definition of $\erre_{\parallel}$, the inner product $\left\langle \xn_{\kmax},\erre_{\parallel}\right\rangle =\left\langle \xn_{\kmax},\erre\right\rangle =\left\langle \xn_{\kmax},\left(\bm{I}-\P_{\hats}\right)\err\right\rangle  $. This random variable is equal in distribution to a Gaussian random variable $Z\sim\Nor\left( 0,\left\Vert \left(\bm{I}-\P_{\hats}\right)\xn_{\kmax}\right\Vert _{2}^{2}\right) $. 
	By Eq. \eqref{eq:proj_I_PA}, $\left\Vert \left(\bm{I}-\P_{\hats}\right)\xn_{\kmax}\right\Vert _{2}\geq \sqrt{1-\mud}$.
	As in the proof of Lemma \ref{lem:pmax1}, 
	\begin{equation}\label{eq:PrBK}
	\Pr\left[B\geq T\right]\geq\Phi^c\left(\frac{1-\sqrt{\rhot}}{\gamma_2}\sqrt{2\log d} \right),
	\end{equation}
	where $\gamma_2=\sqrt{1-\mud}\left(1-\mumax\right)$.
	Recall the definition of $\delta$ in Eq. \eqref{eq:delta}. By Eq. \eqref{eq:mud_bound},  $\gamma_2 \geq \sqrt{1-\delta}(1-\mumax)$.
	Combining this with the bounds \eqref{eq:PrAK} and \eqref{eq:PrBK} completes the proof of Lemma \ref{lem:pmax}. 
\end{proof}

\begin{proof}[Proof of Lemma \ref{lem:p_j_t}]
	Fix a non-support index $j\notin \S$. Recall that $p_{j}$, defined in Eq. \eqref{eq:p_j_m}, is the probability
	that index $j$ is selected by \OMPStep.
	This occurs if $j$ has the
	highest correlation with the current residual, i.e.,
	\[
	p_{j}=\Pr\left[\left|\langle\xn_{j},\res\rangle\right|\geq\max_{i\in\left[d\right]\setminus\hats}\left|\langle\xn_{i},\res\rangle\right|\right].
	\]
	Clearly, by taking the maximum over a subset of the indices $\mathcal{A}\subseteq \left[d\right]\setminus\hats$ that includes $j$, the probability can only be higher. Namely,
	\begin{equation}\label{eq:j_geq_noise_t}
	p_{j}\geq\Pr\left[\left|\langle\xn_{j},\res\rangle\right|\geq\max_{i\in{\cal A}}\left|\langle\xn_{i},\res\rangle\right|\right].
	\end{equation}
	Here we take $\mathcal{A}$ as the set of all non-support indices plus the index $\kmax$, i.e., $\mathcal{A}=\left(\left[d\right]\setminus \S\right)\cup\left\lbrace \kmax\right\rbrace  $, where $\kmax$ is defined in Eq. \eqref{eq:kmax}.
	Next, we separately upper bound 
	\begin{equation}\label{eq:j_gtr_i}
	\Pr\left[\left|\left\langle \xn_{j},\res\right\rangle \right|>\max_{i\notin \S\cup\left\{ j\right\} }\left|\left\langle \xn_{i},\res\right\rangle \right|\right]
	\end{equation}
	and
	\begin{equation}\label{eq:j_gtr_kmax}
	\Pr\left[ \left|\left\langle \xn_{j},\res\right\rangle \right|>\left|\left\langle \xn_{\kmax},\res\right\rangle \right|\right] 
	\end{equation}
	and then upper bound $p_j$ using \eqref{eq:ABC1} with $A= \left|\left\langle \xn_{j},\res\right\rangle \right|$, $B=\max_{i\notin\left(\S\cup\left\lbrace j\right)\right\rbrace }\left|\left\langle \xn_{i},\res\right\rangle \right|$, and  $C=\left|\left\langle \xn_{\kmax},\res\right\rangle \right|$.
	
	For later use in both bounds, 
	the random variable $A$ can be upper bounded as follows
	\begin{eqnarray}
	\left|\left\langle \xn_{j},\res\right\rangle \right| & = & \left|\sum_{l\in \S\setminus\hats}\tn_{l}\left\langle \xn_{j},\left(\bm{I}-\P_{\hats}\right)\xn_{l}\right\rangle +\sigma\left\langle \xn_{j},\erre\right\rangle \right|\nonumber\\
	& \leq & \tmaxn\sum_{l\in \S\setminus\hats}\left|\left\langle \xn_{j},\left(\bm{I}-\P_{\hats}\right)\xn_{l}\right\rangle \right|+\sigma\left|\left\langle \xn_{j},\erre\right\rangle \right|
	\leq \tmaxn\Ku\left(\mumax+\mud\right)+\sigma\left|\left\langle \xn_{j},\erre\right\rangle \right|, \label{eq:j_bound}
	\end{eqnarray}
	where the first equality follows from Eq. \eqref{eq:res_m_K}, the next inequality follows from the triangle inequality and the definition of $\tmaxn$ in Eq. \eqref{eq:tmaxn}, and the last inequality follow from \eqref{eq:inprod_I_PA}.
	We now begin with event \eqref{eq:j_gtr_i}. By Eqs. \eqref{eq:inprod_I_PA} and \eqref{eq:tmu_mip}, for each non-support index $i\notin \S$ such that $i\neq j$,
	\begin{eqnarray*}
		\left|\left\langle \xn_{i},\res\right\rangle \right| & \geq & \left\langle \xn_{i},\res\right\rangle 
		=  \sum_{l\in \S\setminus\hats}\tn_{l}\left\langle \xn_{i},\left(\bm{I}-\P_{\hats}\right)\xn_{l}\right\rangle +\sigma\left\langle \xn_{i},\erre\right\rangle \\
		& \geq & -\tmaxn\Ku\left(\mumax+\mud\right)+\sigma\left\langle \xn_{i},\erre\right\rangle 
		.
	\end{eqnarray*}
	Combining the above bound with Eq. \eqref{eq:j_bound}, rearranging the terms, 
	recalling the relation between $\tmaxn$ and $\rho$ in \eqref{eq:tmaxn_K} and applying inequality \eqref{eq:tmu_mip} yields
	\begin{equation}\label{eq:T6}
	\Pr\left[\left|\left\langle \xn_{j},\res\right\rangle \right|>\max_{i\notin \S\cup\left\{ j\right\} }\left|\left\langle \xn_{i},\res\right\rangle \right|\right]\leq\Pr\left[\left|\left\langle \xn_{j},\erre\right\rangle \right|+\frac{2K\mumax\sqrt{2\rhot\log d}}{1-(2K-1)\mumax}>\max_{i\notin \S\cup\left\{ j\right\} }\left\langle \xn_{i},\erre\right\rangle \right].
	\end{equation}

	As in the proof of Lemma \ref{lem:p_j_t1}, applying \eqref{eq:DET} 
	with $T=(1-\epsilon)\sqrt{2\left(1-\mumax\right)\log d}$ and $\epsilon\in\left(0,1\right)$ as in \eqref{eq:eps_cond_K}, we can upper bound \eqref{eq:T6} by 
	\begin{equation}\label{eq:two_probs}
	\Pr\left[\left|\langle\xn_{j},\erre\rangle\right|\geq a\sqrt{2\log d}\right]+\Pr\left[\max_{i\notin \S\cup\left\{ j\right\} }\langle\xn_{i},\erre\rangle<(1-\epsilon)\sqrt{2\left(1-\mumax\right)\log d}\right],
	\end{equation}
	where $$a=(1-\epsilon)\sqrt{1-\mumax}-\frac{2K\mumax\sqrt{\rhot}}{1-(2K-1)\mumax}.$$
	By the symmetry of projection matrices, $\langle\xn_{j},\erre\rangle=\langle\left(\bm{I}-\P_{\hats}\right)\xn_{j},\err\rangle$.
	By Eq. \eqref{eq:proj_I_PA}, the norm $\left\Vert \left(\bm{I}-\P_{\hats}\right)\xn_{j}\right\Vert _{2} \leq 1$ and thus the first term in \eqref{eq:two_probs} is bounded by 
	\begin{equation}
	\label{eq:first_term}
	2\Phi^{c}\left(a\sqrt{2\log d}\right).
	\end{equation}
	We now bound the second term in \eqref{eq:two_probs} using Lemma \ref{lem:lopes} with $Z_{i}=\left\langle \frac{\left(\I-\P_{\hats}\right)\xn_{i}}{\left\Vert \left(\I-\P_{\hats}\right)\xn_{i}\right\Vert _{2}},\err\right\rangle $. Towards this goal, notice that by Eq. \eqref{eq:proj_I_PA}, $\left\Vert \left(\I-\P_{\hats}\right)\xn_{i}\right\Vert _{2}\geq\sqrt{1-\mud}$. Thus, the second term in \eqref{eq:two_probs} is upper bounded by 
	\begin{equation}\label{eq:pre_lopes}
	\Pr\left[\max_{i\notin \S\cup\left\{ j\right\} }Z_{i}<\frac{(1-\epsilon)\sqrt{2\left(1-\mumax\right)\log d}}{\sqrt{1-\mud}}\right].
	\end{equation}
	Furthermore, by Eqs. \eqref{eq:proj_I_PA} and \eqref{eq:inprod_I_PA}, for each $i,l\notin \S\cup\left\{ j\right\}  $ such that $i\neq l$,
	$
	\E\left[Z_{i}Z_{l}\right]\leq\frac{\mumax+\mud}{1-\mud}.
	$
	Thus, we can apply Lemma \ref{lem:lopes} with $\eta=\frac{\mumax+\mud}{1-\mud}$ and $\zeta=1-\epsilon$ to obtain that  \eqref{eq:pre_lopes} is bounded by 
	\begin{equation}\label{eq:post_lopes}
	C\left(d-K-1\right)^{-\frac{1-\mumax}{\mumax+\mud}\epsilon^{2}}\log^{\left(\frac{1-\mumax}{\mumax+\mud}\epsilon-1\right)/2}\left(d-K-1\right).
	\end{equation}
	Similarly to the proof of Lemma \ref{lem:p_j_t1}, under condition \eqref{eq:eps_cond_K} on $\epsilon$ and for sufficiently large $d=d(\epsilon)$, \eqref{eq:first_term} is larger than \eqref{eq:post_lopes}, and thus 
	\begin{equation}
	\label{eq:p_non_supp}
	\Pr\left[\left|\left\langle \xn_{j},\res\right\rangle \right|>\max_{i\notin \S\cup\left\{ j\right\} }\left|\left\langle \xn_{i},\res\right\rangle \right|\right]\leq 4\Phi^{c}\left(a\sqrt{2\log d}\right).
	\end{equation}

	We now turn to bounding  \eqref{eq:j_gtr_kmax}.
	For the support index $\kmax$, similarly to \eqref{eq:T3K},
	\[
	\left|\left\langle \xn_{\kmax},\res\right\rangle \right|  \geq  \sign\left(\tn_{\kmax}\right)\left\langle \xn_{\kmax},\res\right\rangle 
	\geq  \tmaxn\left(1-\mud\right)-\tmaxn(\Ku-1)\left(\mumax+\mud\right)+\sigma\sign\left(\tn_{\kmax}\right)\left\langle \xn_{\kmax},\erre\right\rangle.
	\]
	Combining the bound above with Eq. \eqref{eq:j_bound}, recalling the relation between $\tmaxn$ and $\rhot$ in \eqref{eq:tmaxn_K} and applying the inequality \eqref{eq:tmu_mip} yields
	\[
	\Pr\left[ \left|\left\langle \xn_{j},\res\right\rangle \right|>\left|\left\langle \xn_{\kmax},\res\right\rangle \right|\right] \leq 
	\Pr\left[ \left|\left\langle \xn_{j},\erre\right\rangle \right|-\sign\left(\tn_{\kmax}\right)\left\langle \xn_{\kmax},\erre\right\rangle >\sqrt{2\rhot\log d}\right].
	\]
	We now upper bound this probability.
	Let $H=\left\langle \xn_{j},\erre\right\rangle =\left\langle \left(\bm{I}-\P_{\hats}\right)\xn_{j},\err\right\rangle $ and $G=\sign\left(\tn_{\kmax}\right)\left\langle \xn_{\kmax},\erre\right\rangle =\sign\left(\tn_{\kmax}\right)\left\langle \left(\bm{I}-\P_{\hats}\right)\xn_{\kmax},\err\right\rangle $. 
	Notice that by Eqs. \eqref{eq:proj_I_PA}, \eqref{eq:inprod_I_PA}, and \eqref{eq:mud_mumax}, $\sigma_{H}^{2},\sigma_{G}^{2}\leq 1$. Combining Eqs. \eqref{eq:delta} and \eqref{eq:mud_bound} implies that $\mud\leq \delta$. Hence, $\left|\sigma_{HG} \right|\leq\mumax+\mud\leq \mumax+\delta$.
	Thus, as in the proof of Lemma \ref{lem:p_j_t1},
	\begin{equation}\label{eq:sig_bound}
	\Pr\left[\left|H\right|-G>c\right] \leq 2\Phi^c\left( \sqrt{\frac{2\rho\log d}{2+2\left( \mumax+\delta\right) }}\right).
	\end{equation}
	By Eq. \eqref{eq:ABC1}, the probability \eqref{eq:j_geq_noise_t} is at most the minimum between \eqref{eq:p_non_supp}
	and \eqref{eq:sig_bound}. By the monotonicity of the Gaussian CDF, \eqref{eq:j_geq_noise_t} is upper bounded by
	\begin{equation}\label{eq:T7_K}
	4\Phi^{c}\left(\max\left\lbrace \sqrt{\frac{\rho}{2+2\left( \mumax+\delta\right)}},\left((1-\epsilon)\sqrt{1-\mumax}-\frac{2K\mumax\sqrt{\rho}}{1-(2K-1)\mumax}\right)\right\rbrace \sqrt{2\log d}\right).
	\end{equation}
	Similarly to the proof of Lemma \ref{lem:p_j_t1}, inserting the
	definitions of $Q_1$ and $Q_2$ in Eqs. \eqref{eq:Q_1_K} and \eqref{eq:Q_2_K}
	respectively, into Eq. \eqref{eq:r_cond_K} 
	and rearranging various terms yields 
	\begin{equation}\label{eq:r_cond2_K}
	\frac{1-\sqrt{r}}{\sqrt{1-\delta}\left( 1-\mumax\right) }+\sqrt{Q_{0}}<\max\left\lbrace
	\sqrt{\frac{r}{2+2\left( \mumax+\delta\right)}},\left((1-\epsilon)\sqrt{1-\mumax}-\frac{2K\mumax\sqrt{r}}{1-(2K-1)\mumax}\right)\right\rbrace.
	\end{equation}
	The definitions of $r$ and $\rho$ in Eqs. \eqref{eq:r_def} and \eqref{eq:rho} imply that $\rhot\geq r$. Thus, $\rho$ satisfies Eq. \eqref{eq:r_cond_K} and hence Eq.
	\eqref{eq:r_cond2_K}.
	The RHS of Eq. \eqref{eq:r_cond2_K} is the same as the
	maximum in Eq. \eqref{eq:T7_K} above.
	Thus, Eq. \eqref{eq:T7_K} is upper bounded by
	\begin{equation}
	4\Phi^c\left( \left( \frac{1-\sqrt{\rhot}}{\sqrt{1-\delta}\left( 1-\mumax\right)} +\sqrt{Q_0}\right)
	\sqrt{2\log d}\right).
	\end{equation}
	By the assumption $\rho\leq 1$, we can apply Lemma \ref{lem:two_term_tail}. Hence, by the definition of
	$Q_0$ in Eq. \eqref{eq:Q_0_K}, and by the definition of $\F$ in Eq. \eqref{eq:F}, 
	\begin{eqnarray*}			
		p_{j} & \leq & 4\Phi^{c}\left(\left(\frac{1-\sqrt{\rhot}}{\sqrt{1-\delta}\left(1-\mumax\right)}+\sqrt{Q_{0}}\right)\sqrt{2\log d}\right)\leq4\sqrt{2}d^{-Q_{0}}\Phi^{c}\left(\frac{1-\sqrt{\rhot}}{\sqrt{1-\delta}\left(1-\mumax\right)}\sqrt{2\log d}\right)\\
		& = & 4\sqrt{2}\frac{1}{88\sqrt{2}K}\Phi^{c}\left(\frac{1-\sqrt{\rhot}}{\sqrt{1-\delta}\left(1-\mumax\right)}\sqrt{2\log d}\right)=\frac{\F\left(d,K,\mumax,\rhot\right)}{11K},
	\end{eqnarray*}
	which completes the proof of Lemma \ref{lem:p_j_t}.
\end{proof}

\subsection{Proofs of technical lemmas}\label{app:tech}
\begin{proof}[Proof of Lemma \ref{lem:two_term_tail}]
	Classical results by \citet{birnbaum1942inequality} and \citet{komatu1955elementary} are
	that for all $x\geq 0$, the following inequalities hold
	\begin{equation}
	\frac{2e^{-x^2/2}}{\sqrt{2\pi}\left( \sqrt{x^2+4}+x\right)
	}<\Phi^{c}\left(x\right)<\frac{2e^{-x^2/2}}{\sqrt{2\pi}\left(
		\sqrt{x^2+2}+x\right) }.
	\end{equation}
	Hence, 
	\[
	\Phi^{c}\left(a+b\right) < \frac{2e^{-(a+b)^2/2}}{\sqrt{2\pi}\left(
		\sqrt{(a+b)^2+2}+a+b\right) },
	\]
	and 
	\[		
	\Phi^{c}\left(a\right)>\frac{2e^{-a^{2}/2}}{\sqrt{2\pi}\left(\sqrt{a^{2}+4}+a\right)}.
	\]
	Combining the two yields the following
	\[	
	\Phi^{c}\left(a+b\right)<\frac{\left(\sqrt{a^{2}+4}+a\right)e^{-ab}}{\left(\sqrt{(a+b)^{2}+2}+a+b\right)}e^{-b^{2}/2}\Phi^{c}\left(a\right).
	\]
	Notice that for any $a\geq 0$,
	the fraction in the above display
	is a decreasing function of $b$. Since $b\geq 0$, it suffices to note that 
	$\frac{\left(\sqrt{a^{2}+4}+a\right)}{\left(\sqrt{a^{2}+2}+a\right)}\leq\sqrt{2}
	$ for any $a\geq 0$.
\end{proof}

Towards proving Lemma \ref{lem:projs}, we prove the following Lemma \ref{lem:proj_i}, which bounds the inner product between vectors projected to the subspace orthogonal to $\Xn_{\mid\hats}$ under the assumption $\hats \subset \S$.
\begin{lemma}\label{lem:proj_i}
	Let $\hats\subset \S$ and denote $\Kd=|\hats|$. 
	Assume that $\left( \Kd-1\right) \mumax<1$. Then, for any pair of vectors $\Ev_1,\Ev_2\in\R^n$
	\begin{equation}\label{eq:inprod}
	\left\langle \Ev_{1},\Ev_{2}\right\rangle -\frac{\left|\sum_{j\in\hats}\left\langle \xn_{j},\Ev_{1}\right\rangle \left\langle \xn_{j},\Ev_{2}\right\rangle \right|}{1-\left(\Kd-1\right)\mumax}\leq\left\langle \Ev_{1},\left(\bm{I}-\P_{\hats}\right)\Ev_{2}\right\rangle \leq\left\langle \Ev_{1},\Ev_{2}\right\rangle +\frac{\left|\sum_{j\in\hats}\left\langle \xn_{j},\Ev_{1}\right\rangle \left\langle \xn_{j},\Ev_{2}\right\rangle \right|}{1-\left(\Kd-1\right)\mumax}.
	\end{equation}
	If in addition $\Ev_1=\Ev_2=\Ev$, then
	\begin{equation}\label{eq:norm}
	\left\Vert \Ev\right\Vert _{2}^{2}-\frac{\sum_{j\in\hats}\left\langle \xn_{j},\Ev\right\rangle ^{2}}{1-\left(\Kd-1\right)\mumax}\leq\left\Vert \left(\bm{I}-\P_{\hats}\right)\Ev\right\Vert _{2}^{2}\leq\left\Vert \Ev\right\Vert _{2}^{2}-\frac{\sum_{j\in\hats}\left\langle \xn_{j},\Ev\right\rangle ^{2}}{1+\left(\Kd-1\right)\mumax}.
	\end{equation}
\end{lemma}

\begin{proof}[Proof of Lemma \ref{lem:proj_i}]
	First, if $\hats=\emptyset$ then clearly $\left\langle \Ev_{1},\left(\bm{I}-\P_{\hats}\right)\Ev_{2}\right\rangle =\left\langle \Ev_{1},\Ev_{2}\right\rangle $ and both \eqref{eq:inprod} and \eqref{eq:norm} trivially hold.
	Therefore, assume that $\Kd\geq 1$. In this case 
	\begin{equation}\label{eq:a1IPSa2}
	\left\langle \Ev_{1},\left(\bm{I}-\P_{\hats}\right)\Ev_{2}\right\rangle =\left\langle \Ev_{1},\Ev_{2}\right\rangle -\left\langle \Ev_{1},\P_{\hats}\Ev_{2}\right\rangle .
	\end{equation}
	By definition of $\P_{\hats}$ in Eq. \eqref{eq:P}, 
	\begin{equation}\label{eq:a1PSa2}
	\left\langle \Ev_{1},\P_{\hats}\Ev_{2}\right\rangle =\left\langle \Ev_{1},\Xn_{\mid\hats}\left(\Xn_{\mid\hats}^{T}\Xn_{\mid\hats}\right)^{-1}\Xn_{\mid\hats}^{T}\Ev_{2}\right\rangle =\left\langle \Xn_{\mid\hats}^{T}\Ev_{1},\left(\Xn_{\mid\hats}^{T}\Xn_{\mid\hats}\right)^{-1}\Xn_{\mid\hats}^{T}\Ev_{2}\right\rangle .
	\end{equation}
	We now bound this term in absolute value.
	For a matrix $\EM\in \R^{n\times n}$, denote by $\lmin\left( \EM \right) $ and $\lmax\left( \EM \right) $ its minimal and maximal eigenvalues, respectively.
	Consider $\EM=\Xn_{\mid\hats}^{T}\Xn_{\mid\hats} $.
	Each of its entries $\EM_{i,j}$ is an inner product $\left\langle \xn_{i},\xn_{j}\right\rangle $ where $i,j\in \hats$. Hence, all of its diagonal entries are $1$ and all of its off-diagonal entries are bounded in absolute value by $\mumax$. 
	By the Gershgorin circle theorem, the eigenvalues of $\EM$ lie in the interval $1\pm (\Kd-1)\mumax$. Since $\left( \Kd-1\right) \mumax<1$, all eigenvalues are strictly positive. Thus $\EM$ is invertible, and the eigenvalues of $\EM^{-1}$ satisfy 
	\begin{equation}\label{eq:evals}
	\frac{1}{1+\left(\Kd-1\right)\mumax}\leq\lmin\left(\EM^{-1}\right)\leq\lmax\left(\EM^{-1}\right)\leq\frac{1}{1-\left(\Kd-1\right)\mumax}.
	\end{equation}
	Since the eigenvalues of $\EM^{-1}$ are strictly positive, for any pair of vectors $\bm{u},\bm{v}\in\R^n$, 
	\[
	\lmin\left(\EM^{-1}\right)\left|\left\langle \bm{u},\bm{v}\right\rangle \right|\leq\left|\left\langle \bm{u},\EM^{-1}\bm{v}\right\rangle \right|\leq\lmax\left(\EM^{-1}\right)\left|\left\langle \bm{u},\bm{v}\right\rangle \right|.
	\]
	Inserting $\bm{u}=\Xn_{\mid\hats}^{T}\Ev_1$, $\bm{v}=\Xn_{\mid\hats}^{T}\Ev_2$ and Eq. \eqref{eq:a1PSa2} yields
	\[
	\lmin\left(\EM^{-1}\right)\left|\left\langle \Xn_{\mid\hats}^{T}\Ev_{1},\Xn_{\mid\hats}^{T}\Ev_{2}\right\rangle \right|\leq\left|\left\langle \Ev_{1},\P_{\hats}\Ev_{2}\right\rangle \right|\leq\lmax\left(\EM^{-1}\right)\left|\left\langle \Xn_{\mid\hats}^{T}\Ev_{1},\Xn_{\mid\hats}^{T}\Ev_{2}\right\rangle \right|.
	\]
	Combining the bounds in Eq. \eqref{eq:evals} with the decomposition $\Xn_{\mid\hats}\Xn_{\mid\hats}^{T}=\sum_{j\in\hats}\xn_{j}\xn_{j}^{T}$ gives
	\begin{equation}\label{eq:a1PSa2_final}
	\frac{\left|\sum_{j\in\hats}\left\langle \xn_{j},\Ev_{1}\right\rangle \left\langle \xn_{j},\Ev_{2}\right\rangle \right|}{1+\left(\Kd-1\right)\mumax}\leq\left|\left\langle \Ev_{1},\P_{\hats}\Ev_{2}\right\rangle \right|\leq\frac{\left|\sum_{j\in\hats}\left\langle \xn_{j},\Ev_{1}\right\rangle \left\langle \xn_{j},\Ev_{2}\right\rangle \right|}{1-\left(\Kd-1\right)\mumax}.
	\end{equation}
	When $\Ev_{1}\neq \Ev_{2}$, the term $\left\langle \Ev_{1},\P_{\hats}\Ev_{2}\right\rangle $ can have an arbitrary sign. Thus, inserting the upper bound into Eq. \eqref{eq:a1IPSa2} proves the inequality \eqref{eq:inprod}.

	We now prove the inequality \eqref{eq:norm}.
	Let $\Ev_1=\Ev_2=\Ev$.
	Since each of the terms in the two sums in Eq. \eqref{eq:a1PSa2_final} is positive, we may remove the absolute values, i.e., 
	\begin{equation}\label{eq:aPSa}
	\frac{\sum_{j\in\hats}\left\langle \xn_{j},\Ev\right\rangle ^{2}}{1+\left(\Kd-1\right)\mumax}\leq\left\langle \Ev,\P_{\hats}\Ev\right\rangle \leq\frac{\sum_{j\in\hats}\left\langle \xn_{j},\Ev\right\rangle ^{2}}{1-\left(\Kd-1\right)\mumax}.
	\end{equation}
	Recall that since $\left(\bm{I}-\P_{\hats}\right)$ is a projection matrix, it is symmetric and idempotent. Thus,
	\begin{equation}\label{eq:I_PS}
	\left(\bm{I}-\P_{\hats}\right)^{T}\left(\bm{I}-\P_{\hats}\right)=\left(\bm{I}-\P_{\hats}\right)\left(\bm{I}-\P_{\hats}\right)=\left(\bm{I}-\P_{\hats}\right).
	\end{equation}
	Hence,
	\[
	\left\Vert \left(\bm{I}-\P_{\hats}\right)\Ev\right\Vert _{2}^{2}=\left\langle \Ev,\left(\bm{I}-\P_{\hats}\right)\Ev\right\rangle =\left\Vert \Ev\right\Vert ^{2}-\left\langle \Ev,\P_{\hats}\Ev\right\rangle .
	\]
	Inserting inequality \eqref{eq:aPSa} completes the proof of \eqref{eq:norm} and of Lemma \ref{lem:proj_i}.
\end{proof}

\begin{proof}[Proof of Lemma \ref{lem:projs}]
	We begin with proving inequalities \eqref{eq:mud_mumax} and \eqref{eq:tmu_mip}.
	By the max-MIP condition \eqref{eq:max_mip_cond}, $1-\left(2K-1\right)\mumax>0$. Rearranging implies that $$1-\left(K-2\right)\mumax>\left(K+1\right)\mumax.$$ Combining this with the bound on $\mud$ in \eqref{eq:mud_bound} gives 
	\[
	\mud\leq \frac{K-1}{K+1}\mumax\leq \mumax,
	\]
	which proves \eqref{eq:mud_mumax}.
	The max-MIP condition \eqref{eq:max_mip_cond} implies that $1-\left(K-1\right)\mumax>0.$
	Using $\Kd+\Ku=K$ and rearranging yields $\frac{\Ku\mumax}{1-(\Kd-1)\mumax}<1$. Combining the definition of $\mud$ in \eqref{eq:mud} with this bound implies that
	\[
	\Ku\mud=\Kd\frac{\Ku\mumax^{2}}{1-(\Kd-1)\mumax}<\Kd\mumax.
	\]
	Hence,
	\[
	K\mumax=\Ku\mumax+\Kd\mumax>\Ku\mumax+\Ku\mud,
	\]
	which proves \eqref{eq:tmu_mip}.
	
	We now prove the remaining inequalities using Lemma \ref{lem:proj_i}, beginning with \eqref{eq:proj_I_PA}.
	Since $\P_{\hats}$ is a projection matrix, for any index $i\notin \hats$, 
	\[
	\left\Vert \left(\bm{I}-\P_{\hats}\right)\xn_{i}\right\Vert _{2}^{2}\leq\left\Vert \xn_{i}\right\Vert _{2}^{2}=1.
	\]
	Recall that for any distinct pair of indices $i\neq j $, it holds that $0\leq \left\langle \xn_{j},\xn_{i}\right\rangle^2\leq \mumax^2$.
	By Eq. \eqref{eq:norm} with $\Ev=\xn_i$,
	\begin{equation*}\label{eq:i_PS_i}
	\left\Vert \left(\bm{I}-\P_{\hats}\right)\xn_{i}\right\Vert _{2}^{2}\geq1-\frac{\sum_{j\in\hats}\left\langle \xn_{j},\xn_{i}\right\rangle ^{2}}{1-\left(\Kd-1\right)\mumax}\geq1-\frac{\Kd\mumax^{2}}{1-\left(\Kd-1\right)\mumax}=1-\mud,
	\end{equation*}
	which concludes the proof of \eqref{eq:proj_I_PA}.

	Next, we prove inequality \eqref{eq:inprod_I_PA}.
	By the right inequality in \eqref{eq:inprod} with $\Ev_1=\xn_k$ and $\Ev_2=\xn_i$, 
	\[
	\left\langle \xn_{k},\left(\bm{I}-\P_{\hats}\right)\xn_{i}\right\rangle \leq\left\langle \xn_{k},\xn_{i}\right\rangle +\frac{\left|\sum_{j\in\hats}\left\langle \xn_{j},\xn_{k}\right\rangle \left\langle \xn_{j},\xn_{i}\right\rangle \right|}{1-\left(\Kd-1\right)\mumax}.
	\]
	Thus, by the triangle inequality and by the definitions of $\mumax$ and $\mud$ in Eqs. \eqref{eq:mumax} and \eqref{eq:mud} respectively,
	\begin{eqnarray*}
		\left|\left\langle \xn_{k},\left(\bm{I}-\P_{\hats}\right)\xn_{i}\right\rangle \right| & \leq & \left|\left\langle \xn_{k},\xn_{i}\right\rangle \right|+\frac{\left|\sum_{j\in\hats}\left\langle \xn_{j},\xn_{k}\right\rangle \left\langle \xn_{j},\xn_{i}\right\rangle \right|}{1-\left(\Kd-1\right)\mumax}\\
		& \leq & \mumax+\frac{\Kd\mumax^{2}}{1-\left(\Kd-1\right)\mumax}=\mumax+\mud.
	\end{eqnarray*}

	Finally, we prove inequality \eqref{eq:proj_I_Pk_PA}.
	Recall that by the max-MIP condition \eqref{eq:max_mip_cond}, $\mumax<\frac{1}{K-1}$.
	For any distinct pair of indices $i\neq k $ such that $i,k\notin\hats$,
	Eq. \eqref{eq:norm} with $\Ev=\left(\I-\P_k \right) \xn_i$ gives
	\begin{eqnarray*}
		\left\Vert \left(\bm{I}-\P_{\hats}\right)\left(\bm{I}-\P_{k}\right)\xn_{i}\right\Vert _{2}^{2} & \geq & \left\Vert \left(\bm{I}-\P_{k}\right)\xn_{i}\right\Vert _{2}^{2}-\frac{\sum_{j\in\hats}\left\langle \xn_{j},\left(\bm{I}-\P_{k}\right)\xn_{i}\right\rangle ^{2}}{1-\left(\Kd-1\right)\mumax}\\
		& = & 1-\left\langle \xn_{k},\xn_{i}\right\rangle ^{2}-\frac{\sum_{j\in\hats}\left(\left\langle \xn_{j},\xn_{i}\right\rangle -\left\langle \xn_{k},\xn_{i}\right\rangle \left\langle \xn_{j},\xn_{k}\right\rangle \right)^{2}}{1-\left(\Kd-1\right)\mumax}\\
		& \geq & 1-\mumax^{2}-\frac{\Kd\left(\mumax+\mumax^{2}\right)^{2}}{1-\left(\Kd-1\right)\mumax}\\
		&= & 1-\mumax^{2}-\mud(1+\mumax)^2,
	\end{eqnarray*}
	which concludes the proof of Eq. \eqref{eq:proj_I_Pk_PA}.
	It remains to prove that $$1-\mumax^{2}-\mud(1+\mumax)^2>0.$$
	First, let $K=1$. This implies that $\hats=\emptyset$ and thus $$\left\Vert \left(\bm{I}-\P_{\hats}\right)\left(\bm{I}-\P_{\kmax}\right)\xn_{i}\right\Vert _{2}=\left\Vert \left(\bm{I}-\P_{\kmax}\right)\xn_{i}\right\Vert _{2}\geq1-\mumax^{2},$$ which is positive by the max-MIP condition \eqref{eq:max_mip_cond}.
	Now let $K>1$. By the max-MIP condition \eqref{eq:max_mip_cond}, $\mumax<1$ and $\frac{\Kd\mumax}{1-\left(\Kd-1\right)\mumax}<1$.
	Thus,
	\begin{eqnarray*}
		1-\mumax^{2}-\mud(1+\mumax)^{2} & > & 1-\mumax^{2}-\mumax\left(1+\mumax\right)^{2}\\
		& = & 1-\mumax\left(1+3\mumax+\mumax^{2}\right)\\
		& > & 1-\mumax\left(1+4\mumax\right).
	\end{eqnarray*} 
	Note that for each $K>1$, it holds that $\mumax<\frac{K-1}{2}$. 
	Thus,
	\[
	1-\mumax\left(1+4\mumax\right)>1-\mumax\left(1+2K-2\right)>0,
	\]
	where the last inequality is another application of the max-MIP condition \eqref{eq:max_mip_cond}.		
\end{proof}

\section{ADDITIONAL SIMULATION RESULTS}\label{app:add_exp}
Theorem \ref{thm:DJ-OMP} holds under the max-MIP condition \eqref{eq:max_mip_cond} and assumptions \ref{assum:M}-\ref{assum:r}.
However, in practice, \DJ succeeds even if these assumptions are not met. 
For example, the max-MIP condition does not hold in the setting used in Figure \ref{fig:succ}(b), and thus none of the additional assumptions hold either.
To examine assumption \ref{assum:M} further, we performed the following simulation, whose results are depicted in Figure \ref{fig:M_vs_d}.
As described in Section \ref{sec:sim}, we generated matrices with i.i.d. Gaussian entries, i.e., $\alpha=0$, with a fixed number of samples $n=2000$, varying dimension $d$, varying number of machines $M$, and varying sparsity level $K$. 
In each simulation, the noise level is $\sigma=1$, and each of the $K$ nonzero values of the sparse vector $\tv$ equals $\tmin=0.06$.
We then used linear extrapolation to estimate for each dimension the number of machines needed to reach a given success probability, in our example $0.5$, and displayed them on a logarithmic scale. In addition, we display a least-squares-based linear estimation of the relation between $\log(M)$ and $\log(d)$. The small resulting sum of squared residuals (SSR) support our result that the relationship is of the form $M=O(d^{\beta})$ for some $0<\beta<1$, even when the max-MIP condition does not hold, and in fact $\beta$ is empirically smaller than the exponent derived in Eq. \eqref{eq:Mc_K-maintext}. 
In addition, the estimated number of machines increases with $K$, which is also in accordance with Eq. \eqref{eq:Mc_K-maintext}. 
We obtained similar results when the matrices were slightly correlated, with slightly higher estimated number of machines.

\begin{figure}[!t]
	\centering
	\includegraphics[width=0.75\columnwidth]{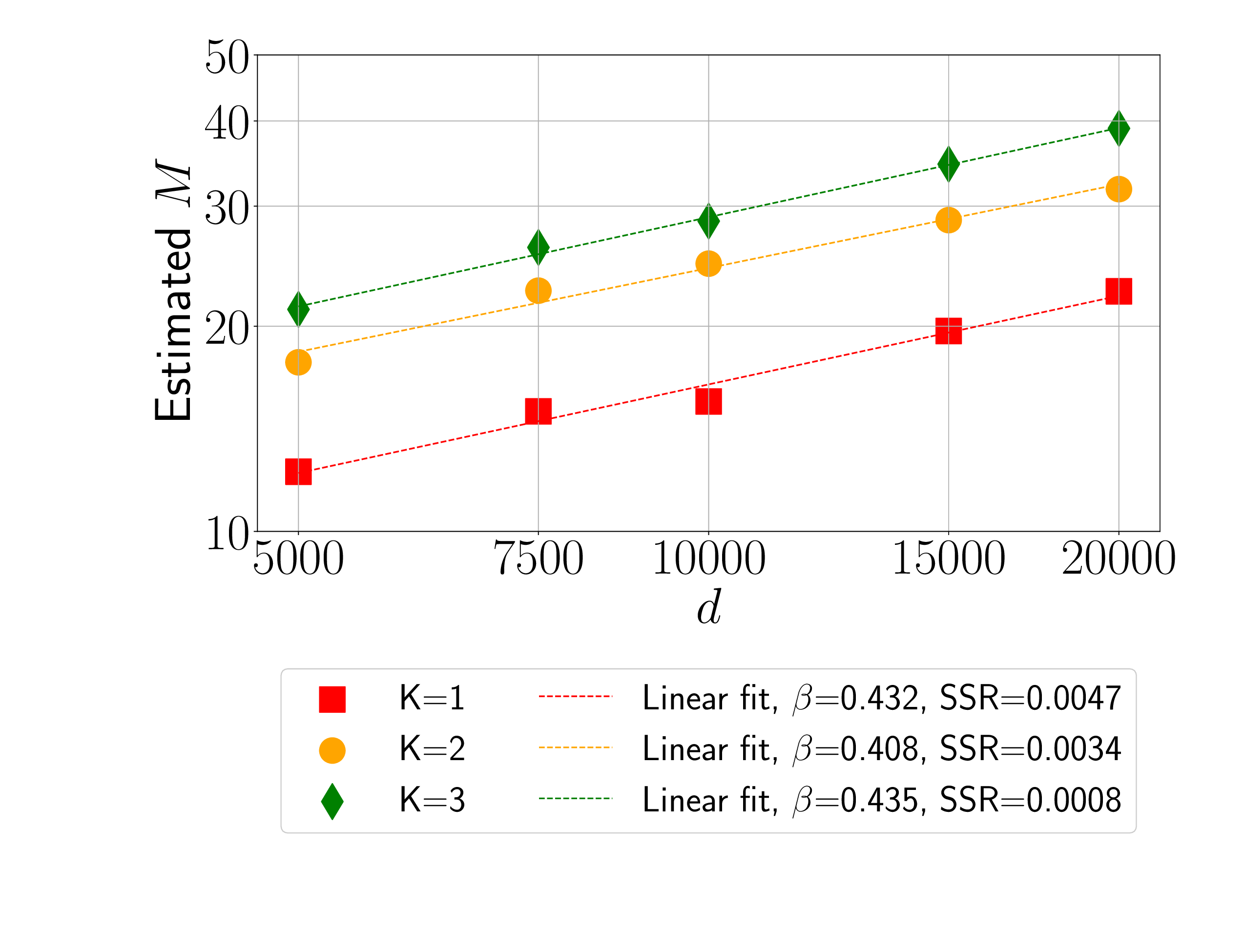}
	\caption{Number of Machines for Support Recovery by \DJ vs. Dimension}
	\label{fig:M_vs_d}
\end{figure}

\section{IMPLEMENTATION DETAILS}\label{sec:implmnt}
The code used to generate the simulations in Section \ref{sec:sim} was implemented in Python and was executed on an internal cluster \citep[v3.8;][PSF licensed]{python38}.
For \SIS-based methods, we used the SIS package by \citet{saldana2018sis}, which was implemented using R statistical software \citep[v4.0.3;][]{R2023} and embedded into the Python code using the rpy2 package (https://rpy2.github.io/), all licensed by GPL-2 licenses.
\texttt{Lasso}-based methods were implemented using the scikit-learn package by \citet[BSD License]{scikit}. 
Other libraries that were used include NumPy \citep[liberal BSD license]{harris2020array}, SciPy \citep[BSD license]{2020SciPy}, and Matplotlib \citep[BSD compatible license]{hunter2007matplotlib}.

\end{document}